\documentclass{article}

\usepackage{microtype}
\usepackage{graphicx}
\usepackage{subfigure}
\usepackage{booktabs} 

\newlength{\without}

\usepackage{multirow}
\usepackage[thinc]{esdiff}
\usepackage{amsthm}
\usepackage{amssymb}
\newtheorem{definition}{Definition}
\newtheorem{theorem}{Theorem}
\newtheorem{lemma}{Lemma}
\newtheorem{property}{Property}

\newtheorem{assumption}{Assumption}
\newtheorem{corollary}{Corollary}
\usepackage{mathtools}
\DeclarePairedDelimiter\ceil{\lceil}{\rceil}
\DeclarePairedDelimiter\floor{\lfloor}{\rfloor}
\usepackage{hyperref}

\usepackage[accepted]{icml2021}

\begin{document}

\twocolumn[
\icmltitle{Bayesian Optimistic Optimisation with Exponentially Decaying Regret}




\begin{icmlauthorlist}
\icmlauthor{Hung Tran-The}{to}
\icmlauthor{Sunil Gupta}{to}
\icmlauthor{Santu Rana}{to}
\icmlauthor{Svetha Venkatesh}{to}
\end{icmlauthorlist}

\icmlaffiliation{to}{Applied Artificial Intelligence Institute, Deakin University, Geelong, Australia}

\icmlcorrespondingauthor{Hung Tran-The}{hung.tranthe@deakin.edu.au}

\icmlkeywords{Machine Learning, ICML}

\vskip 0.3in
]



\printAffiliationsAndNotice{} 

\begin{abstract}
 Bayesian optimisation (BO) is a well-known efficient algorithm for finding the global optimum of expensive, black-box functions. The current practical BO algorithms have regret bounds ranging from $\mathcal{O}(\frac{logN}{\sqrt{N}})$ to $\mathcal O(e^{-\sqrt{N}})$, where $N$ is the number of evaluations. This paper explores the possibility of improving the regret bound in the noiseless setting by intertwining concepts from BO and tree-based optimistic optimisation which are based on partitioning the search space. We propose the BOO algorithm, a first practical approach which can achieve an exponential regret bound with order $\mathcal O(N^{-\sqrt{N}})$ under the assumption that the objective function is sampled from a Gaussian process with a Mat\'ern kernel with smoothness parameter $\nu > 4 +\frac{D}{2}$, where $D$ is the number of dimensions. We perform experiments on optimisation of various synthetic functions and machine learning hyperparameter tuning tasks and show that our algorithm outperforms baselines.
\end{abstract}
\section{Introduction}
We consider a global optimisation problem whose goal is to maximise $f(x)$ subject to $x \in \mathcal X \subset \mathbb{R}^D$, where $D$ is the number of dimensions and $f$ is an expensive black-box functions that can only be evaluated point-wise. The performance of a global optimisation algorithm is typically evaluated using \emph{simple regret}, which is given as
$$r_N = \text{sup}_{x \in \mathcal X} f(x) - \text{max}_{1 \le i \le N}f(x_i)$$
where $x_i$ is the $i$-th sample, and $N$ is the number of function evaluations. In this paper, we consider the case that the evaluation of $f$ is noiseless.

Bayesian optimisation (BO) provides an efficient model-based solution for global optimisation. The core idea is to transform a global optimisation problem into a sequence of auxiliary optimisation problems of a surrogate function called the acquisition function. The acquisition function is built using a model of the function through its limited observations and recommends the next function evaluation location. Regret analysis has been done for many existing BO algorithms, and typically the regret is sub-linear following order $\mathcal{O}\left(\sqrt{\frac{logN}{N}}\right)$ \citep{Srinivas12,RussoRKOW18}. More recently, \citep{vakili20} have improved this to $\mathcal{O}\left(\frac{1}{\sqrt{N}}\right)$ under the noiseless setting. However, a limitation of BO is that performing such a sequence of auxiliary optimisation problems is expensive.

\citet{Freitas12} introduced a Gaussian process (GP) based scheme called $\delta$-cover sampling as an alternative to the acquisition function to trade-off exploration and exploitation. Their method samples the objective function using a finite lattice within a feasible region and doubles the density of points in the lattice at each iteration. However, even in moderate dimensions, their algorithm is impractical since the lattice quickly becomes too large to be sampled in a reasonable amount of time (as pointed out by \citep{WangSJF14,Kawaguchi16}).

An alternative practical approach for global optimisation is to consider tree-based optimistic optimisation as in \citep{munos2011,Floudas}. These algorithms partition the search space into finer regions by building a hierarchical tree. The key is to have efficient strategies to identify a set of nodes that may contain the global optimum and then to successively reduce and refine the search space to reach closer to the optimum.
As example, DIRECT algorithm \citep{Jones93} partitions the search space assuming a global Lipschitz constant. Simultaneous Optimistic Optimisation (SOO) algorithm \citep{munos2011} generalises this by using only local Lipschitz conditions without requiring the knowledge of the Lipschitz global-metric. Under certain assumptions, this algorithm shows the possibility of achieving an exponentially diminishing regret $\mathcal{O}(e^{-\sqrt{N}})$. An additional advantage of such algorithms is that we do not need to perform an auxiliary non-convex optimisation of the acquisition functions as in BO which may be difficult in cases that are high-dimensional \citep{Kandasamy15,Tran-The0RV20} or have unbounded search spaces \citep{hung20}. However, these optimistic algorithms are model-free, that is, they do not utilise the function observations efficiently.

A natural extension to improve the sample efficiency is to incorporate a model of the objective function into the optimistic strategy. Indeed, works that do this include BaMSOO \citep{WangSJF14} and IMGPO \citep{Kawaguchi16}. Using a Gaussian process \citep{Rasmussen05} as a model of the objective function, both algorithms avoid to evaluate the objective function for points known to be sub-optimal with high probability. While BaMSOO has a sub-linear regret bound, IMGPO can achieve an exponential regret bound. However, despite this, IMGPO has not still overcome the worst-case regret bound order $\mathcal{O}(e^{-\sqrt{N}})$ of SOO which does not use any model of the objective function. A natural question is that \emph{is there a practical algorithm for global optimisation that can break this regret bound order $\mathcal{O}(e^{-\sqrt{N}})$ under a mild assumption}?

In this paper, we propose a novel approach, which combines the strengths of the tree-based optimistic optimisation methods and Bayesian optimisation to achieve an improved regret bound $\mathcal{O}(N^{-\sqrt{N}})$ in the worst case. Our main contributions are summarised as follows:
\begin{itemize}
  \item A GP-based optimistic optimisation algorithm using novel partitioning procedure and function sampling;
  \item Our algorithm has a worst-case regret bound of $\mathcal{O}(N^{-\sqrt{N}})$ in the noiseless setting under the assumption that the objective function is sampled from a Gaussian process with a Mat\'ern kernel with smoothness parameter $\nu > 4 + \frac{D}{2}$, where $N$ is the number of evaluations and $D$ is the number of dimensions. Our algorithm avoids an auxiliary optimisation step at each iteration in BO, and avoids the $\delta$-cover sampling in the approach of \citet{Freitas12}.  To our best knowledge, without using an $\delta$-cover sampling procedure which is impractical, this is the tightest regret bound for BO algorithms;
  \item To validate our algorithm in practice, we perform experiments on optimisation of various synthetic functions and machine learning hyperparameter tuning tasks and show that our algorithm outperforms baselines.
\end{itemize}
\section{Related Works}
In this section, we briefly review some related work additional to the work mentioned in section 1.

In Bayesian optimisation literature, there exist also some works that use tree-structure for the search space. While \citet{wang18c} used a Mondrian tree to partition the search space, a recent work by \citep{wang2020} used a dynamic tree via $K$-means algorithm. However, these works focused on improving BO's performance empirically for large-scale data sets or high-dimensions rather than to improve the regret bound.

There are two viewpoints for BO, Bayesian and non-Bayesian as pointed out by \citet{scarlett17a}. In the non-Bayesian viewpoint, the function is treated as fixed and unknown, and assumed to lie in a reproducing kernel Hilbert space (RKHS). Under this viewpoint, \citet{chowdhury17a,janz20a}  provided upper regret bounds while \citet{scarlett17a} provided lower regret bounds for BO with Mat\'ern kernels. These bounds all are sub-linear. Otherwise, in the Bayesian viewpoint where we assume that the underlying function is random according to a GP, \citet{Kawaguchi16} showed that BO can obtain an exponential convergence rate. In this paper, we focus on the Bayesian viewpoint and break the regret bound order of IMGPO \citep{Kawaguchi16} under some mild assumptions.

The optimistic optimisation methods have also been extended to adapt to different problem settings e.g., noisy setting \citep{ValkoCM13,Grill15},  high dimensional spaces \citep{QianY16,Al-DujailiS17}, multi-objective optimisation \citep{Al-DujailiS18} or multi-fidelity black-box optimisation \citep{sen18a}. Our work can be complementary to these works and the integration of our solution with them may be promising to improve their regret bounds.
\section{Preliminaries}
\paragraph{Bayesian Optimisation} The standard BO routine consists of two key steps: estimating the black-box function from observations and maximizing an acquisition function to suggest next function evaluation point.

Gaussian process is a popular choice for the first step. Formally, we have $f(x)\sim\mathcal{GP}(m(x),k(x,x'))$ where $m(x)$ and $k(x,x')$ are the mean and the covariance (or kernel) functions. Given a set of observations $\mathcal{D}_{1:p}=\{x_{i},y_{i}\}_{i=1}^{p}$ under a noiseless observation model $y_{i}=f(x_{i})$, the predictive distribution can be derived as $P(f(x)|\mathcal{D}_{1:p},x)=\mathcal{N}(\mu_{p+1}(x),\sigma_{p+1}^{2}(x))$, where $\mu_{p+1}(x)=\textbf{k}^{T}K{}^{-1}\textbf{y}+m(x)$ and $\sigma_{p+1}^{2}(x)=k(x,x)-\text{\textbf{k}}^{T}K^{-1}\textbf{k}$. In the above expression we define $\textbf{k }=[k(x,x_{1}),...,k(x,x_{p})]^T$, $K=[k(x_{i},x_{j})]_{1\le i,j\le p}$ and $\textbf{y}=[y_{1},\ldots,y_{p}]$.

Some well-known popular acquisition functions for the second step include upper confidence bound (GP-UCB)\citep{Srinivas12}, expected improvement (EI) \citep{Bull11}, Thompson sampling (TS) \citep{RussoRKOW18} and predictive entropy search (PES) \citep{Lobato14}.
Among them, GP-UCB is given as $\mathcal{U}_{p}(x)=\mu_{p}(x)+\beta_{p}^{1/2}\sigma_{p}(x)$, where $\beta_{p}$ is the parameter balancing between the exploration and exploitation. We will use GP-UCB in our tree expansion scheme to determine the node to be expanded.

In this paper, we focus on the popular class of Mat\'ern kernels for Gaussian process which is defined as
$$k_{\nu}(x, x') = \frac{\sigma^2}{\Gamma(\nu)2^{\nu -1}}(\frac{||x -x'||_2}{\lambda})^{\nu}\mathcal B_{\nu}(\frac{||x- x'||_2}{\lambda}),$$
where $\Gamma$ denotes the Gamma function, $\mathcal{B}_{\nu}$ denotes the modified Bessel function of the second kind, $\nu$ is a parameter controlling the smoothness of the function and $\sigma^{2},\lambda$ are hyper-parameters of the kernel. We assume that the hyper-parameters are fixed and known in advance. However, our work can also be extended for the unknown hyper-parameters of the Mat\'ern kernel as in \citep{vakili20} (for Bayesian setting). Important special cases of $\nu$ include $\nu=\frac{1}{2}$ that corresponds to the exponential kernel and $\nu \rightarrow \infty$ that corresponds to the squared exponential kernel. The Mat\'ern kernel is of particular practical significance, since it offers a more suitable set of assumptions for the modeling and optimisation of physical quantities (\citep{stein1999}).
\paragraph{Hierarchical Partition} We use the hierarchical partition of the search space as in \citep{munos2011}. Given a branch factor $m$, for any depth $h$, the search space $\mathcal{X}$ is partitioned into a set of $m^{h}$ sets $A_{h,i}$ (called cells), where $0\le i\le m^{h}-1$. This partitioning is represented as a $m$-ary tree structure where each cell $A_{h,i}$ corresponds to a node $(h,i)$. A node $(h,i)$ has $m$ children nodes, indexed as $\{(h+1,i_{j})\}_{1\le j\le m}$. The children nodes $\{(h+1,i_{j}),1\le j\le m\}$ form a partition of the parent's node $(h,i)$. The root of the tree corresponds to the whole domain $\mathcal{X}$. The center of a cell $A_{h,i}$ is denoted by $c_{h,i}$ where $f$ and its upper confidence bound is evaluated.
\section{Proposed BOO algorithm}
\subsection{Motivation}
Most of tree-based optimistic optimisation algorithms like SOO, StoSOO \citep{ValkoCM13}, BaMSOO and IMGPO face a \emph{strict negative correlation} between the branch factor $m$ and the number of tree expansions given a fixed function evaluation budget $N$. On the one hand, using a larger $m$ makes a tree finer, which helps to reach closer to the optimum. On the other hand, having more expansions in the tree also allows to create finer partitions in multiple regions of the space. Thus both a larger branch factor and a larger number of tree expansions allow an algorithm to get closer to the optimum. However, each time a node is expanded, the algorithms such as SOO, StoSOO spend $m$ function evaluations - one for each of the $m$ children and thus the number of tree expansions is restricted to at most $\left\lfloor \frac{N}{m}\right\rfloor$. Thus when $m$ increases, the number of tree expansions decreases. We call this phenomenal the \emph{strict negative correlation} of tree-based optimistic optimisation algorithms.

Using the assumption that the objective function is sampled from a GP prior, BaMSOO and IMGPO reduce this negative correlation by evaluating the function only at the children where the UCB value is greater than the best function value observed thus far ($f^{+}$). However, the number of expansions is still tied to the branch factor lying between $\left\lfloor \frac{N}{m}\right\rfloor$  and $N$.

We present a new approach which permits to untie the branch factor $m$ from the number of tree expansions and hence, solves the strict negative correlation of tree-based optimistic optimisation algorithms. By doing so, we can exploit the use of a large $m$ to achieve finer
partitions and achieves a regret bound $\mathcal{O}(N^{-\sqrt{N}})$ improving upon current BO algorithms.
\subsection{BOO algorithm}
Our algorithm is described in Algorithm \ref{alg:alg} where we assume that the objective function is a sample from GP as in Bayesian optimisation, however our approach follows the principle of SOO which uses a hierarchical partitioning of the search space $\mathcal{X}$. The main difference of the proposed BOO and previous works lies in the partitioning procedure, the tree expansion mechanism and the function sampling strategy.
\begin{algorithm}[tb]
\caption{The BOO Algorithm} \label{alg:alg}
\textbf{Input}: An evaluation budget $N$ and parameters $m, a ,b \in \mathbb{N}$. \\
\textbf{Initialisation}: Set $\mathcal T_0 = \{(0,0)\}$ (root node). Set $p = 1$. Sample initial points to build $\mathcal D_{0}$.\\
\begin{algorithmic}[1]
\WHILE{True}
    \STATE Set $v_{max} = - \infty$
    \FOR{$h =0$ to \text{min}(\text{depth}($\mathcal T_p$), $h_{max}(p)$)}
        \STATE Among all leaves $(h,j)$ of depth $h$, select $(h,i) \in \text{argmax}_{(h,j) \in \mathcal L} \mathcal U_p(c_{h,j})$         \IF{$\mathcal U_p(c_{h,i}) \ge v_{max}$}
            \STATE \textcolor{blue}{Expand node} $(h,i)$ by adding $m$ children $(h+1, i_j)$ to tree $\mathcal T_p$, using partitioning procedure $P(m;a,b)$
            \STATE \textcolor{blue}{Evaluate} $f(c_{h,i})$
            \STATE Augment the data $\mathcal D_{p} = \{\mathcal D_{p-1}, ((c_{h,i}, f(c_{h,i}))\}$.
            \STATE Fit the Gaussian process using $\mathcal D_{p}$
            \STATE Update $v_{max} = \text{max}\{f(c_{h,i}), v_{max}\}$
            \STATE Update $p = p + 1$
            \IF{$p = N$} \STATE Return $x(N) = \text{argmax}_{\{c_{h,i})| (c_{h,i}, f(c_{h,i}))\in \mathcal D_N \} } f(c_{h,i})$
            \ENDIF
        \ENDIF
    \ENDFOR
\ENDWHILE
\end{algorithmic}
\end{algorithm}
\begin{figure}[t]
\centering
\subfigure{\includegraphics[scale=1.0,width=.22\textwidth]{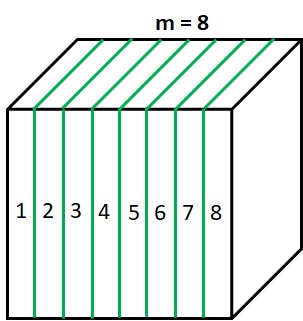}
}\quad
\subfigure{\includegraphics[scale=1.0,width=.22\textwidth]{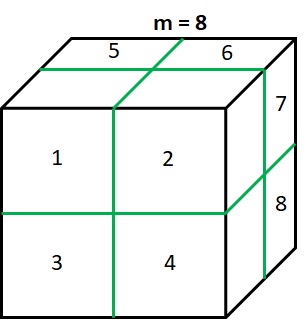}
}
\vspace*{-3mm}
\caption{An illustration of partitioning procedure for $m=8$: (a) SOO methods partition $a$ cell by dividing the longest side into $m$ equal parts; (b) our method sets $m=a^{b}$ and partitions $a$ cell by dividing $b=3$ longest sides into $a=2$ equal parts.}
\label{partition-procedure}
\vspace*{-3mm}
\end{figure}
\begin{figure}[tb]
\begin{center}
\centerline{\includegraphics[width=\columnwidth]{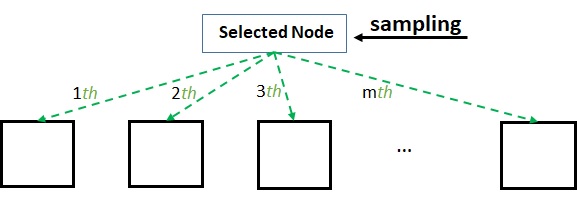}}
\vspace*{-3mm}
\caption{SOO samples the function at all $m$ children nodes while our sampling strategy samples the function only at the parent node (the node selected for expansion). As a result, our strategy requires only one function evaluation irrespective of the value of $m$.}.
\label{sampling}
\end{center}
\vspace*{-6mm}
\end{figure}
\paragraph{Partitioning Procedure} Unlike SOO based algorithms including BaMSOO and IMGPO which often divide a cell into $m$ children cells along the longest side of the cell, we use a novel partitioning procedure which exploits the particular decomposition of the branch factor $m$. Given any $a\ge 2$, $1\le b\le D$ and $m=a^{b}$, where $a,b,m \in \mathbb{N}$, our partitioning procedure, denoted by $P(m;a,b)$, divides the cell along its $b$ longest dimensions into $a$ new cells (see Figure \ref{partition-procedure}). When $b=1$ then our procedure becomes simply the traditional partitioning procedure as in SOO. When $b=D$, all dimensions of the cell are divided which benefits our algorithm. We will explain this further in our convergence analysis.
\paragraph{Tree Expansion Mechanism} The algorithm incrementally builds a tree $\mathcal{T}_{p}$ starting with a node $\mathcal{T}_{0}=\{(0,0)\}$ for $p=1...N$, where $N$ is the evaluation budget. At depth $h$, among all the leaf nodes, denoted by $\mathcal L$ of the current tree, the algorithm selects the node with the maximum GP-UCB value, defined as $\mathcal{U}_{p}(c)=\mu_{p}(c)+\beta_{p}^{1/2}\sigma_{p}(c)$, where $\beta_{p}^{1/2} = \sqrt{2log(\pi^{2}p^{3}/3\eta)}$ and $\eta\in(0,1)$. The tree is expanded by adding $m$ children nodes to the selected node. To force the depth of tree after $p$ expansions, we use a function $h_{max}(p)$ which is also a parameter of the algorithm.
We note that the algorithm uses GP-UCB acquisition function to determine the candidate node, but it only performs a maximisation on a finite, discrete set comprising the leaf nodes at the depth in consideration.
\paragraph{Function Sampling Strategy} Once the node is selected for expansion, unlike previous works that evaluate the objective function at children nodes, we propose to evaluate the objective function only at that node without evaluating the function at its children (see Figure \ref{sampling}).

By this sampling scheme, our algorithm allows to untie the branch factor $m$ from the number of tree expansions. As a result, it allows the use of a large $m$ to achieve finer partitions and to reach closer to the optimum. In fact, using a small $m$ as in previous optimistic optimisation methods also reaches finer partitions, however it needs a large number of tree expansions, and thus still needs a large number of evaluations. Consider an example where $m=2^D$ with $a=2$ and $b=D$. Using the partitioning procedure $P(m;2,D)$, our algorithm partitions a node into $m$ cells with the same granularity. To reach the same granularity as our method, SOO algorithm can use the partitioning with $m=2$ and repeat it $D$ times. However, by this way, SOO always spends $2^{D}$ function evaluations while our algorithm only uses one evaluation. BamSOO and IMGPO algorithms have the similar problem although they improve over SOO - they only evaluate the function at nodes $c$ that satisfy the condition $\mathcal{U}(c)>f^{+}$ (best function value observed thus far) is satisfied. In summary, to reach the same granularity as our method, these algorithms need to spend $m'$ evaluations, where $1\le m'\le2^{D}$ depending on the number of nodes c satisfying the condition $\mathcal{U}(c)>f^{+}$. In contrast, our algorithm only spends one evaluation for all cases regardless of the value of $m$. Together with our partitioning procedure, we leverage this benefit to improve the regret bound for optimisation.
\section{Convergence Analysis}
In this section, we theoretically analyse the convergence of our algorithm. We start with assumptions about function $f$.
\subsection{Assumptions}
To guarantee the correctness of our algorithm, we use the following assumptions.
\begin{assumption}
The function $f$ sampled from $\text{GP}(0,k_{\nu})$, that is a zero mean GP with a Mat\'ern kernel $k_{\nu}$ with $\nu> 4 + \frac{D}{2}$, where $D$ is the number of dimensions.
\end{assumption}
\begin{assumption}
The objective function $f$ has a unique global maximum $x^{*}$.
\end{assumption}
The assumption of a unique maximiser holds with probability one in most non-trivial cases \citep{Freitas12}. Under such assumptions, we obtain the following property.
\begin{property} Assume that the function $f$ is sampled from GP($0, k_{\nu}$) satisfying the Assumption 1 and 2. Then,
\begin{enumerate}
\item $f(x^*) - f(x) \le L_1||x^* - x||^2_2$ for every $x \in \mathcal X$, for some constant $L_1 > 0$,
\item $f(x) \le f(x^*) - L_2||x -x^*||^2$ for every $x \in \mathcal B(x^*, \theta)$ for some constants $L_2, \theta > 0$,
\item $f(x^*) - \text{max}_{x \in \mathcal X   \setminus \mathcal B(x^*, \theta)} f(x) > \epsilon_0$ for some $\epsilon_0 > 0$.
\end{enumerate}
\label{as:5}
\end{property}
We note that all constants $L_1, L_2$ and  $\epsilon_{0}$ are \emph{unknown}. The Property 1.2 can be considered as the quadratic behavior of the objective function in the neighborhood of the optimum. This property holds for every Mat\'ern kernel with $\nu > 2$ as argued by \citep{Freitas12} and \citep{WangSJF14}. The result closest to ours is that of IMGPO \citep{Kawaguchi16} that achieves an exponential regret bound $e^{-\sqrt{N}}$. However, their work requires a type of quadratic behavior in the whole search space (as represented in Assumption 2 in their paper) which is quite strong. Compared to it, our assumption is weaker which only requires the quadratic behavior of the function in a neighborhood of the optimum. Despite this, we will show that our algorithm can improve their regret bound.
\subsection{Convergence Analysis}
For the theoretical guarantee, we follow the principle of the optimism in the face of uncertainty as in \citep{munos2011}. The basic idea is to construct the set of expandable nodes at each depth $h$, called the expansion set. We do this in Section 5.2.1. Quantifying the size of the expansion set is a key step in this principle. We do this in Section 5.2.2. Finally, by using upper bounds on the size of these sets, we derive the regret bounds in Section 5.2.3. \textbf{All proofs
are provided in the Supplementary Material}.
\subsubsection{The expansion set}
\begin{definition}
Let the expansion set at depth $h$ be the set of all nodes that could be potentially expanded before the optimal node at depth $h$ is selected for expansion in Algorithm 1. Formally,
$$I_{h}=\{(h,i)| \exists h \le p \le N : \mathcal{U}_{p}(c_{h,i})\ge f(x^{*})-\delta(h;a,b) \}$$,
where $\delta(h;a,b)$ is defined as $$\delta(h;a,b)=L_1Da^{-2\floor{\frac{bh}{D}}},$$ and $\mathcal{U}_{p}(c_{h,i})$ is the upper confidence bound at the center $c_{h,i}$ of node $(h,i)$ after $p$ expansions.
\end{definition}
We note that even though this definition uses $\delta(h;a,b)$ that depends on the unknown metric $L_1$, our BOO algorithm does not need to know this information. The reason we use $\delta(h;a,b)=L_1Da^{-2\floor{\frac{bh}{D}}}$ lies in the following three observations of our partitioning procedure $P(m;a,b)$ in the search space $\mathcal X$. We assume here that $\mathcal X = [0,1]^D$ (this can always be achieved by scaling).
\begin{lemma} Given a cell $A_{h,i}$ at depth $h$, we have that
\begin{enumerate}
  \item the longest side of cell $A_{h,i}$ is at most $a^{-\floor{\frac{bh}{D}}}$, and
  \item the smallest side of cell $A_{h,i}$ is at least $a^{-\ceil{\frac{bh}{D}}}$.
\end{enumerate}
\end{lemma}
\begin{lemma}
Given a cell $A_{h,i}$ at depth $h$, then we have that $\text{sup}_{x\in A_{h,i}}||x-c_{h,i}|| \le  D^{1/2}a^{-\floor{\frac{bh}{D}}},$ where $c_{h,i}$ is the center of cell $A_{h,i}$.
\end{lemma}
\begin{proof}
By Lemma 1, the longest side of a cell at depth $h$ is at most $a^{-\floor{\frac{bh}{D}}}$. Therefore, $\text{sup}_{x\in A_{h,i}}||x-c_{h,i}||\le\sqrt{Da^{-2\floor{\frac{bh}{D}}}}=D^{1/2}a^{-\floor{\frac{bh}{D}}}$.
\end{proof}
We denote a node $(h,i^{*})$ as the \emph{optimal node} at depth $h$ if $x^{*}$ belongs to the cell $A_{h,i^{*}}$.
\begin{lemma}
At a depth $h$, we have that $f(c_{h,i^{*}}) \ge f(x^*) - \delta(h;a,b).$
\end{lemma}
\begin{proof} By Property 1, $f(x^{*})-f(c_{h,i^{*}})\le L_1||x^{*}-c_{h,i^{*}}||^{2}$. By Lemma 2, $||x^{*}-c_{h,i^{*}}||\le D^{1/2}a^{-\floor{\frac{bh}{D}}}=\sqrt{\delta(h;a,b)}/L_1$. Thus, $f(x^{*})-f(c_{h,i^{*}})\le L_1||x^{*}-c_{h,i^{*}}||^{2}\le\delta(h;a,b)$.
\end{proof}

By Lemma 3 and the fact that $\mathcal{U}_{p}(c_{h,i})\ge f(c_{h,i})$ with high probability, we have that $\mathcal{U}_{p}(c_{h,i})\ge f(x^*) - \delta(h;a,b)$ with high probability. It deduces that the global maximum $x^*$ belongs to the expansion set at every depth with high probability.

The expansion set at a depth $h$ in our approach differs from the ones in works of \citep{munos2011,WangSJF14,Kawaguchi16} which is defined as $I_{h}=\{(h,i) | f(c_{h,i})\ge f(x^{*})-\delta(h;a,b)\}$. More precisely, set $I_{h}$ of their works is a smaller set than the set $I_{h}$ in our work defined above because we have $\mathcal{U}_{p}(c_{h,i})\ge f(c_{h,i})$ with high probability. This bigger $I_{h}$ directly might involve unnecessary explorations and therefore, the algorithm may incur higher regret than that of BamSOO and IMGPO. However, we solve this challenge by leveraging our new partitioning procedure $P(m;a,b)$ with $b=D$, and some results from \citep{kanagawa2018,vakili20} which are presented in the following section.
\subsubsection{An upper bound on the size of the expansion set}
To quantify $|I_h|$, we use a concept called the near-optimality dimension as in \citep{munos2011}. In our context, we define the near-optimality dimension as follows:
\begin{definition}
The near-optimality dimension is defined as the smallest $d \ge 0$ such that there exists $C > 0$ such that for any $\delta(h;a,b)$, the maximal number of disjoint balls the with largest size in a cell at depth $h$ with center in $\mathcal{X}_{\delta(h;a,b)}$ is less than $C(\delta(h;a,b))^{d}$, where $\mathcal{X}_{\delta(h;a,b)}=\{x\in\mathcal{X}|\text{ } \mathcal{U}_{p}(x)\ge f(x^{*})-\delta(h;a,b)\}$.
\end{definition}
With partitioning procedure $P(m;a,b)$ where $a=\mathcal{O}(N^{1/D})$ with $b=D$, we will show $d=0$ (which is equivalent to prove $|I_{h}|\le C$) through the following Theorem 1.
\begin{theorem}[Bound on Expansion Set Size]
Consider a partitioning procedure $P(m;a,b)$ where $a=\mathcal{O}(N^{1/D})$ and $b=D$. Then there exist constants $N_{1}>0$ and $C>0$ such that for every for $N\ge N_{1}$ and $h\ge2$, we have, $|I_{h}|\le C$.
\end{theorem}
To prove Theorem 1, we estimate a bound on variance function $\sigma_{p}$ in terms of the function $\delta(h-1;a,b)$ through the following lemma.
\begin{lemma}
Assuming that node $(h,i)$ at the depth $h \ge 1$ is evaluated at the $p$-th evaluation, where $p\ge h$. Thus,
$$\sigma_{p}(c_{h,i})\le C_1(\delta(h-1;a,b))^{\nu/2-D/4},$$
where $C_1$ is a constant.
\end{lemma}
This lemma holds by applying some results about the closeness between the samples from a GP (in Bayesian setting) and the elements of an RKHS (in non-Bayesian setting) as in \citep{kanagawa2018,vakili20}, to the structured search space as in our approach. This technique is novel compared to BaMSOO and IMGPO's ones. We refer to our Supplementary Material in Section 3.

Using this result and the condition $\nu > 4 + D/2$, we achieve a constant bound on the size of set $I_h$.
\subsubsection{Bounding the simple regret}
Next we use the upper bound on $|I_{h}|$ at every depth $h$ to derive a bound on the simple regret $r_{N}$.

Let us use $h_{p}^{*}$ to denote the depth of the deepest expanded node in the branch containing $x^{*}$ after $p$ expansions. Similar to the lemma 2 in \citep{munos2011}, we can bound the sum of $|I_h|$ as follows.
\begin{lemma}
Assume that $f(c)\le\mathcal{U}(c)$ for all centers $c$ of optimal nodes at all depths $0\le h\le h_{max}(p)$ after $p$ expansions. Then for any depth $0\le h\le h_{max}(p)$, whenever $p\ge h_{max}(p)\sum_{i=0}^{h}|I_{i}|$, we have $h_{p}^{*}\ge h$.
\end{lemma}
We use $A_{p}$ to denote the set of all points evaluated by the algorithm and all centers of optimal nodes of the tree $\mathcal{T}_{p}$ after $p$ evaluations.
\begin{lemma}
Pick a $\eta\in(0,1)$. Set $\beta_{p}=2log(\pi^{2}p^{3}/3\eta)$ and $\mathcal L_p(c) = \mu_p(c) - \beta^{1/2}_p\sigma_p(c)$. With probability $1-\eta$, we have
$$\mathcal L_p(c) \le f(c) \le \mathcal U_p(c),$$
for every $p \ge 1$ and for every $c\in A_{p}$.
\end{lemma}
We now use Lemmas 3-6 to derive a simple regret for the proposed algorithm. Here, the simple regret $r_p$ after $p$ expansions is defined as $r_p = f(x^*) - \text{max}_{1 \le i \le p}f(x_i)$, where $x_i$ is the $i$-th sample.
\begin{theorem}[Regret Bound]
Assume that there is a partitioning procedure $P(m;a,b)$ where $a = \mathcal{O}(N^{1/D})$, $b=D$ and $2 \le m  < \sqrt{N}-1$. Let the depth function $h_{max}(p) = \sqrt{p}$. We consider $m^2 < p \le N$, and define $h(p)$ as the smallest integer $h$ such that $ h \ge \frac{\sqrt{p}-m-1}{C} + 2,$
where $C$ is the constant defined by Theorem 1. Pick a $\eta \in(0,1)$. Then for every $N \ge N_1$, the loss is bounded as
\begin{eqnarray*}
r_{p} & \le & \delta(\text{min}\{h(p), \sqrt{p}+1\};a,b) + \\
& + & 4C_1\beta^{1/2}_p(\delta(\text{min}\{h(p)-1, \sqrt{p}\};a,b))^{\nu/2-D/4},
\end{eqnarray*}
with probability $1-\eta$, where $N_1$ is the constant defined in Theorem 1, $C_1$ is the constant defined in lemma 4 and $\beta_N = \sqrt{2log(\pi^{2}N^{3}/3\eta)}$.
\end{theorem}
\begin{proof}
By Theorem 1, the definition of $h(p)$ and the facts that $|I_0| = 1$ and $|I_1| \le m$, we have
\begin{eqnarray*}
\sum_{l=0}^{h(p)-1} |I_l| & = & |I_0| + |I_1| + (|I_2| + ... + |I|_{h(p)-1}) \\
& \le & 1 + m + C(h(p)-2) \le \sqrt{p}
\end{eqnarray*}
Therefore, $\sum_{l=0}^{h(p)-1}|I_{l}|\le \sqrt{p}$. By Lemma 5 when $h(p)-1\le h_{max}(p)= \sqrt{p}$, we have $h_{p}^{*}\ge h(p)-1$. If  $h(p)-1 > \sqrt{p}$ then $h_{p}^{*}= h_{max}(p) = \sqrt{p}$ since the BOO algorithm does not expand nodes beyond depth $h_{max}(p)$. Thus, in all cases, $h_{p}^{*}\ge\text{min}\{h(p)-1, \sqrt{p}\}$.

Let $(h,j)$ be the deepest node in $\mathcal{T}_{p}$ that has been expanded by the algorithm up to $p$ expansions. Thus $h\ge h_{p}^{*}$. By Algorithm 1, we only expand a node when its GP-UCB value is larger than $v_{max}$ which is updated at Line 10 of Algorithm 1. Thus, since the node $(h,j)$ has been expanded, its GP-UCB value is at least as high as that of the some node $(h_{p}^{*}+1, j)$ at depth $h_{p}^{*}+1$, such that (1) node $(h_{p}^{*}+1, o)$ has been evaluated at some $p'$-th expansion before node $(h,j)$  and (2)  $(h_{p}^{*}+1, o) \in  \text{argmax}_{(h_p^{*} +1,i) \in \mathcal L} \mathcal U_{p'}(c_{h_p^{*} +1,i})$ (see Line 4 of Algorithm 1).

Thus, by using Lemma 3 and Lemma 6, we can achieve with probability $1 -\eta$ that $f(x^{*})-U_{p}(c_{h,j})\le \delta(h_p^{*}+1;a,b) + 2\beta^{1/2}_{p'}\sigma_{p'}(c_{h_p^{*} +1,o})$.

Further by Lemma 6, we have $U_p(c_{h,j}) =  \mu_p(c_{h,j}) + \beta^{1/2}_p\sigma_{p}(c_{h,j}) =  \mathcal L_p(c_{h,j}) + 2\beta_p^{1/2}\sigma_{p}(c_{h,j}) \le f(c_{h,j}) + 2\beta_p^{1/2}\sigma_{p}(c_{h,j})$, with a probability $1 -\eta$.

Combining these two results, we have $f(x^*) - f(c_{h,j}) \le \delta(h_p^{*}+1;a,b) + 2\beta^{1/2}_{p'}\sigma_{p'}(c_{h_p^{*} +1,o}) + 2\beta_p^{1/2}\sigma_{p}(c_{h,j})$ with a probability $1 -\eta$.

\begin{figure*}[ht]
\centering
\subfigure{\includegraphics[scale=1.0,width=.32\textwidth, height= .13\textheight]{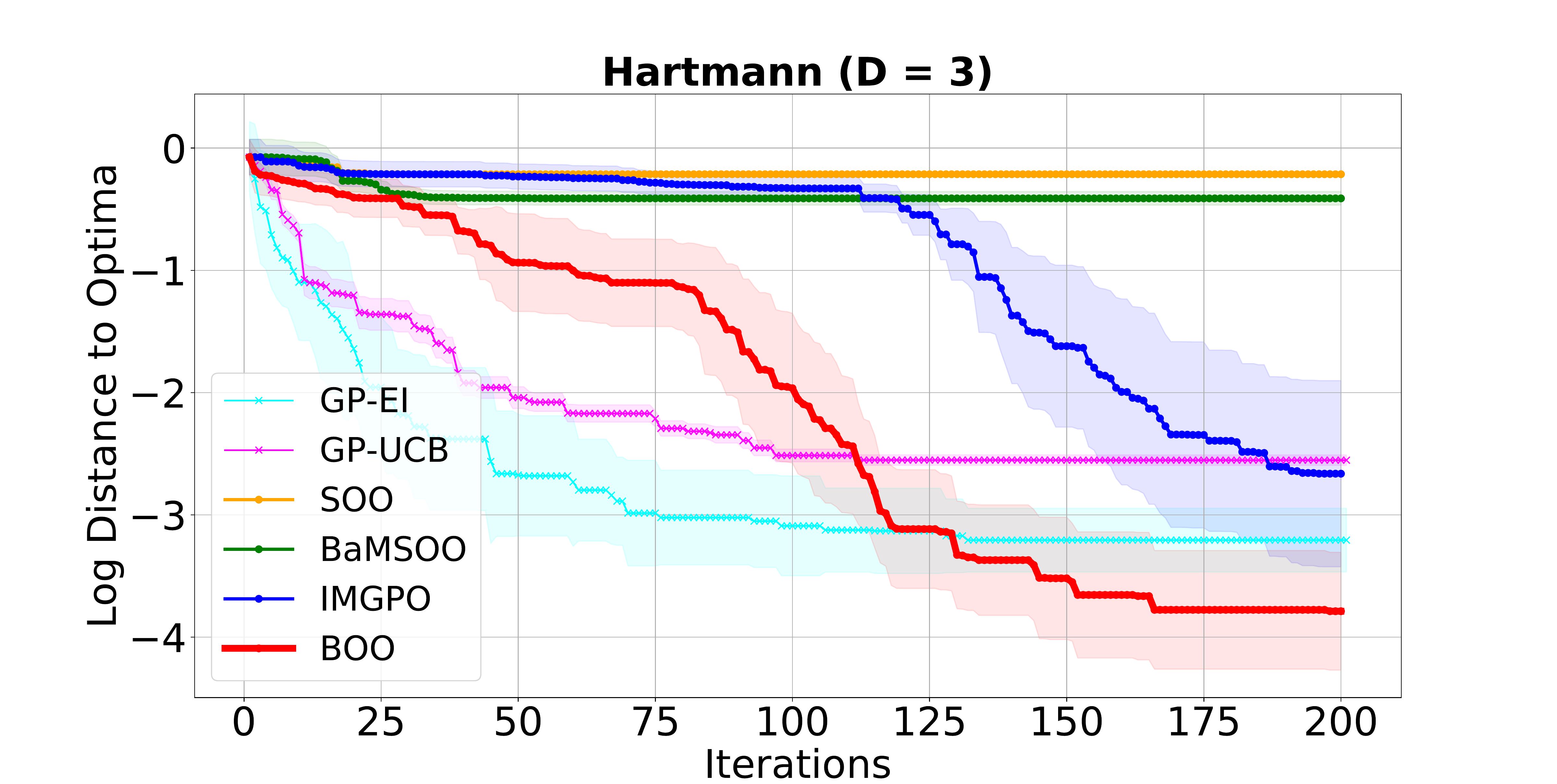}
}\hfill
\subfigure{\includegraphics[scale=1.0,width=.32\textwidth, height= .13\textheight]{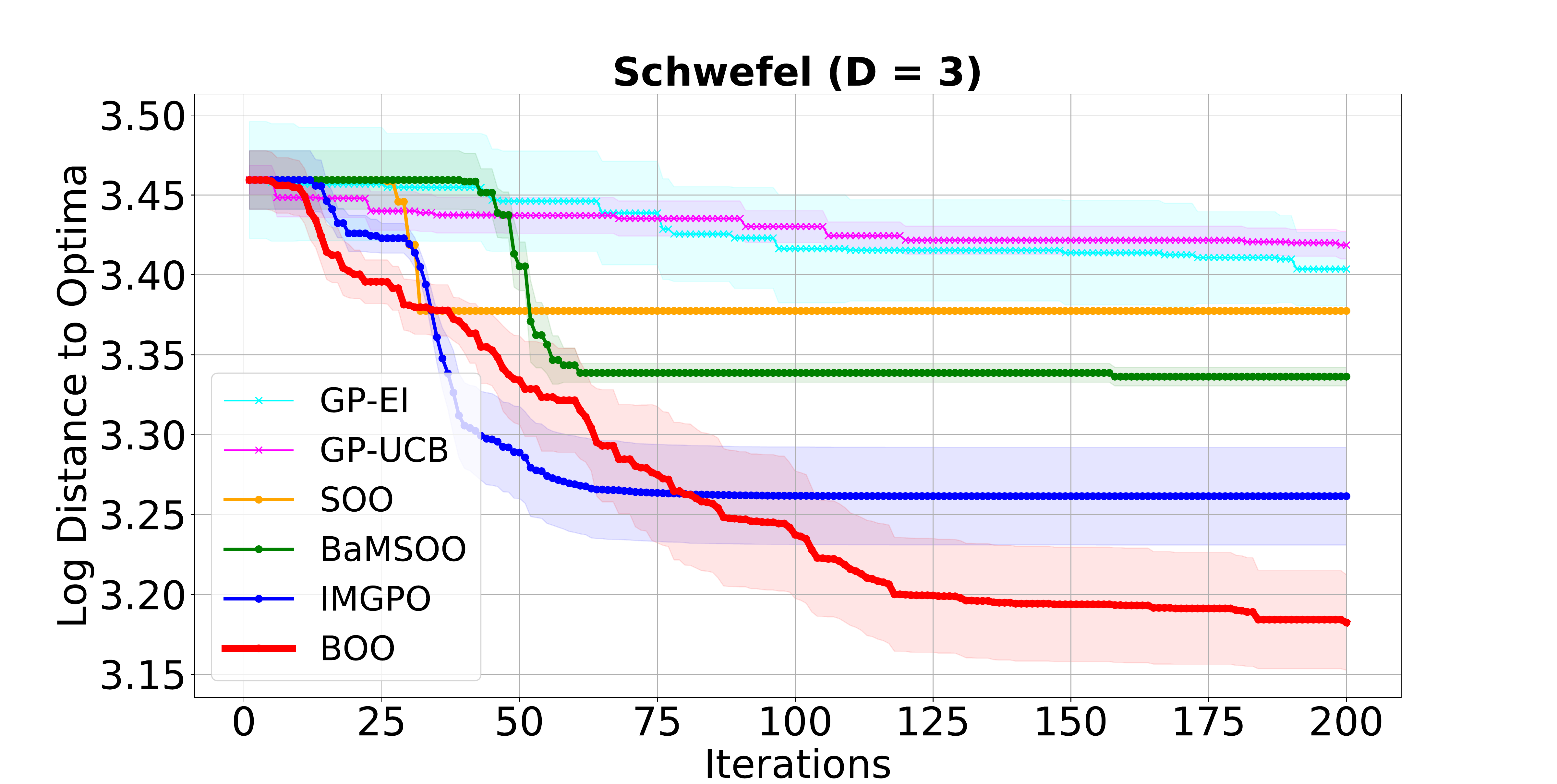}
}\hfill
\subfigure{\includegraphics[scale=1.0,width=.32\textwidth, height= .13\textheight]{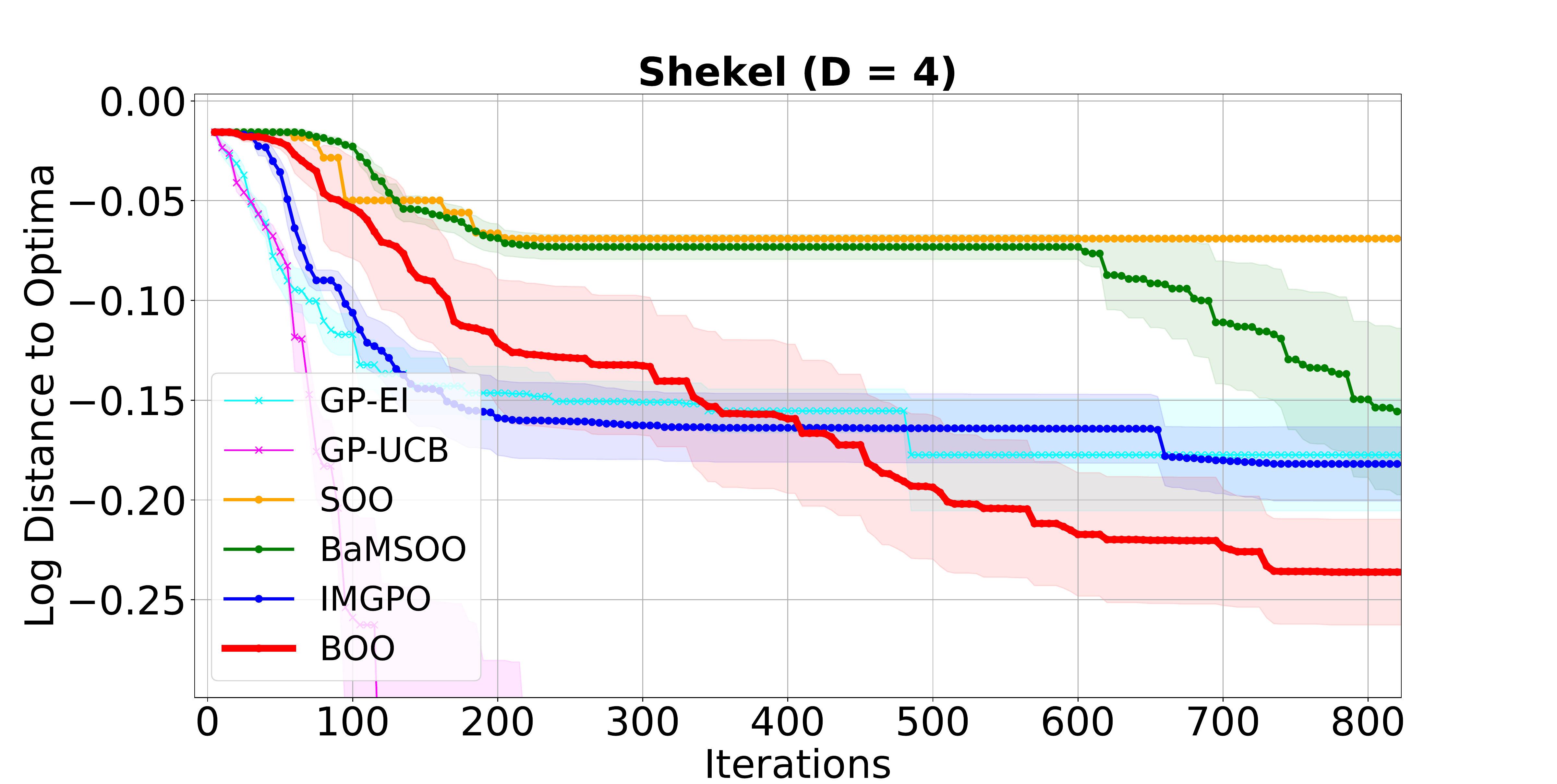}
}
\vspace*{-3mm}
\caption{Comparison of methods for Hartmann3 ($D=3$), Schwefel ($D=3$), and Shekel ($D=4$) functions.}
\label{comparison-of-baseline}
\vspace*{-3mm}
\end{figure*}
Finally, by using Lemma 4 to bound $\sigma_{p'}(c_{h_p^{*} +1,o})$ and $\sigma_{p}(c_{h,j})$ and using the fact that the function $\delta(*;a,b)$ decreases with their depths, we achieve
\begin{eqnarray*}
r_p & \le & f(x^*) - f(c_{h,j}) \\
& \le  & \delta(\text{min}\{h(p), \sqrt{p}+1\};a,b) + \\
& +  & 4C_1\beta^{1/2}_p(\delta(\text{min}\{h(p)-1, \sqrt{p}\};a,b))^{\nu/2-D/4}
\end{eqnarray*}
with a probability $1 -\eta$. \textbf{We provide the complete proof in the Supplementary Material}.
\end{proof}
Finally, we present an improved and simpler expression for the regret bound through the following corollary from Theorem 2.
\begin{corollary}
Pick a $\eta \in(0,1)$. Then, there exists a constant $N_2 > 0$ such that for every $N \ge N_2$, the simple regret of the proposed BOO algorithm with the partitioning procedure $P(m;a,b)$ where $a=\floor{(\frac{\sqrt{N}}{2})^{\frac{1}{D}}}$, $b=D$, is bounded as
$$r_N \le \mathcal O(N^{-\sqrt{N}}),$$ with probability $1-\eta$, where $N$ is the number of sampled points.
\end{corollary}
\paragraph{Remark 1.} A detailed expression for the regret bound of Corollary 1 is that $r_N \leq \mathcal O(C_1L_1D/2^D N^{-\frac{\sqrt{N}}{CD} + \frac{2}{CD} - \frac{2}{D}})$, where $C_1$ is a constant (defined in Lemma 4) given $L_1$ and $D$, and $C$ is a constant (defined in Theorem 1) given $L1, L_2,\eta$ and $\nu$.  This complete formula is extracted from the proof of Corollary 1 in Supplementary Material.
\paragraph{Remark 2.} The closest result to ours is the regret bound of IMGPO which has the worst case order $\mathcal{O}(e^{-\sqrt{N}})$. As can be seen, we have improved the regret bound. Our result improves over previous works because we leverage a large value of the branch factor $m$ and our new partitioning procedure with $b = D$ where all dimensions of a cell are divided.
\section{Experiments}
To evaluate the performance of our BOO algorithm, we performed a set of experiments involving optimisation of three benchmark functions and three real applications. We compared our method against five baselines which have theoretical guarantees: (1) GP-EI \citep{Bull11}, (2) GP-UCB \citep{Srinivas12}, (3) SOO \citep{munos2011}, (4) BaMSOO \citep{WangSJF14}, (5) IMGPO \citep{Kawaguchi16}.
\paragraph{Experimental settings} All implementations are in Python. For each test function, we repeat the experiments 15 times. We plot the mean and a confidence bound of one standard
deviation across all the runs. We used Mat\'ern kernel with $\nu = 4 + (D+1)/2$ which satisfies our assumptions, and estimated the kernel hyper-parameters automatically from data using Maximum Likelihood Estimation. All methods using GP (including GP-EI, GP-UCB, BaMSOO, IMGPO and our method) were started from randomly initialised points to train GP. For GP-EI and GP-UCB which follow the standard BO, we used the DIRECT algorithm to maximise the acquisition functions and computed $\beta_t$ for GP-UCB as suggested in \citep{Srinivas12}. For tree-based space partitioning methods, we follow their implementations to set the branch factor $m$. Note that these methods use a small $m$ due to the negative correlation. SOO and BaMSOO use $m=2$ while IMGPO uses $m=3$. The depth of search tree $h_{max}(p)$ in SOO and BaMSOO was set to $\sqrt{p}$ as suggested in \citep{munos2011,WangSJF14}. The parameter $\Xi_{n}$ in IMGPO was set to 1.
\vspace*{-2mm}
\begin{table}[ht]
\centering
\caption{Average CPU time (in seconds) for the experiment with each test function.}
\resizebox{7cm}{!}{%
\begin{tabular}[t]{lccc}
\hline Algorithm  & Hartmann  & Shekel & Schwefel\\ \hline
GP-EI            & 200.39   & 740.20  & 250.79 \\
GP-UCB           & 880.43   & 1640.87 & 180.96 \\
SOO              & 0.51     & 0.40    & 0.11  \\
BaMSOO           & 39.02    & 87.62  & 28.67 \\
IMGPO            & 23.23    & 80.53   & 34.65 \\
BOO              & 27.21    & 91.22   & 41.01 \\ \hline
\end{tabular}}
\label{Table_Time}
\vspace*{-4mm}
\end{table}%
\subsection{Optimisation of Benchmark Functions}
\begin{figure*}[ht]
\centering
\subfigure{\includegraphics[scale=1.0,width=.32\textwidth, height= .13\textheight]{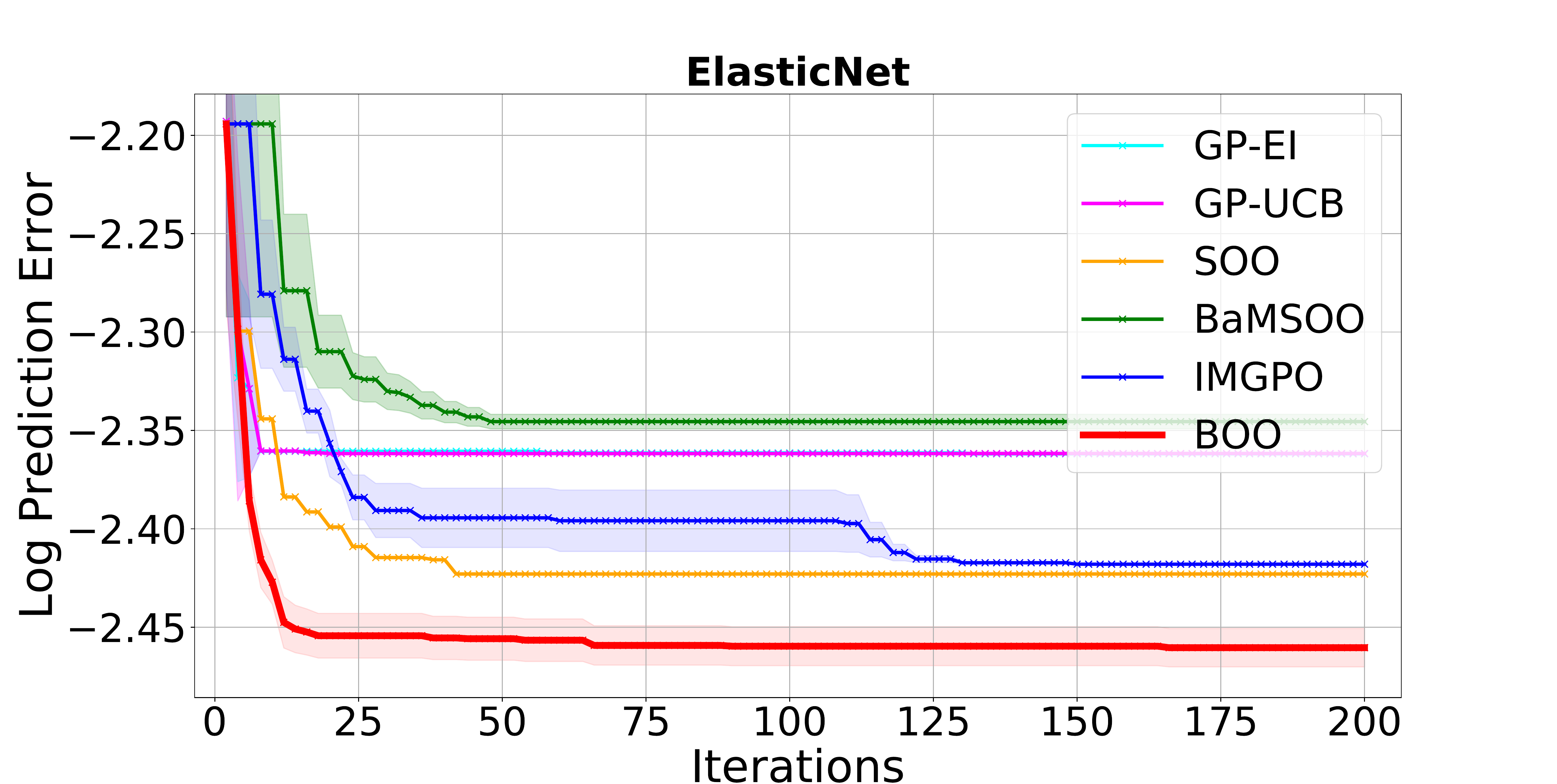}
}\hfill
\subfigure{\includegraphics[scale=1.0,width=.32\textwidth, height= .13\textheight]{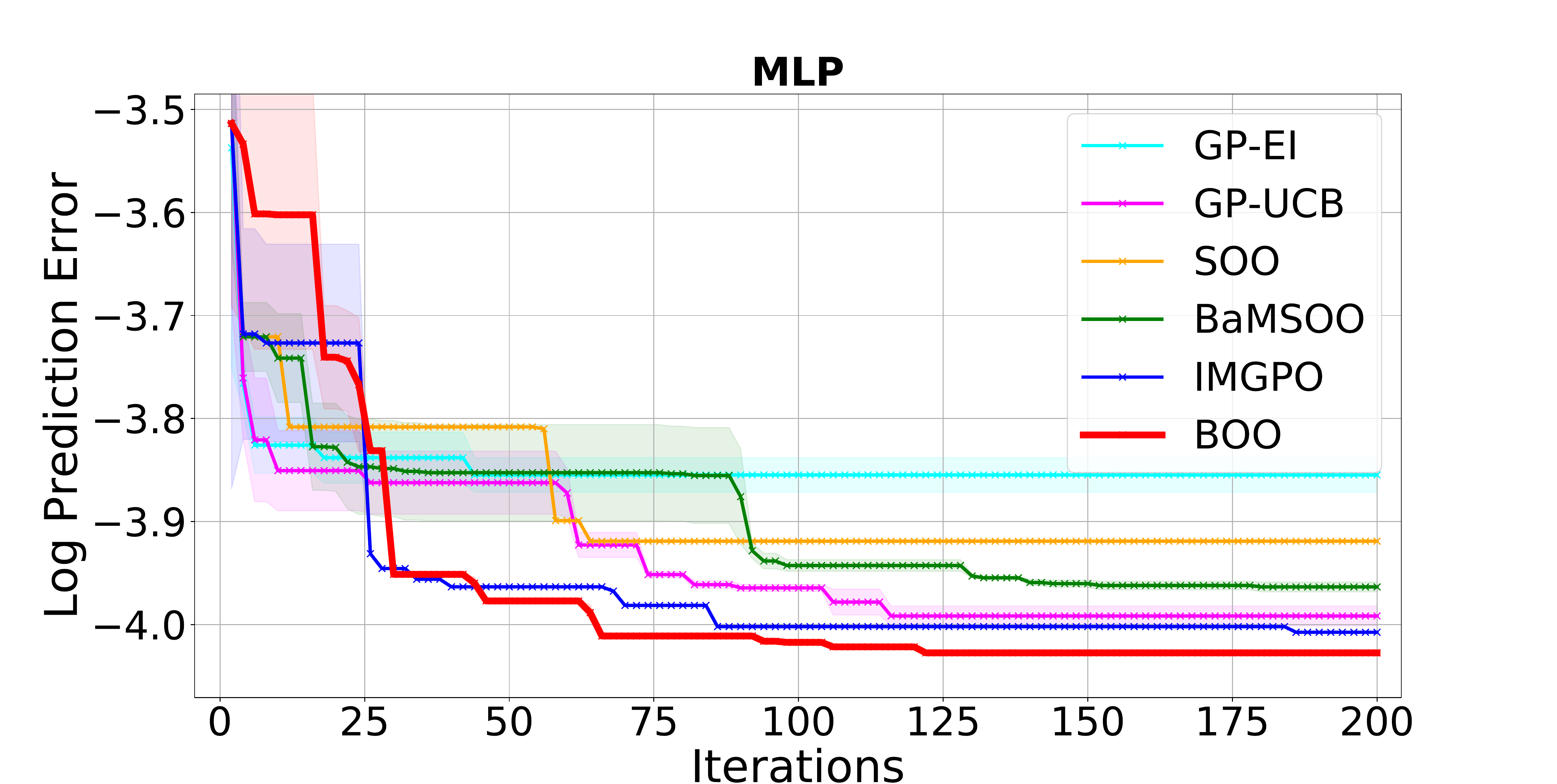}
}\hfill
\subfigure{\includegraphics[scale=1.0,width=.32\textwidth, height= .13\textheight]{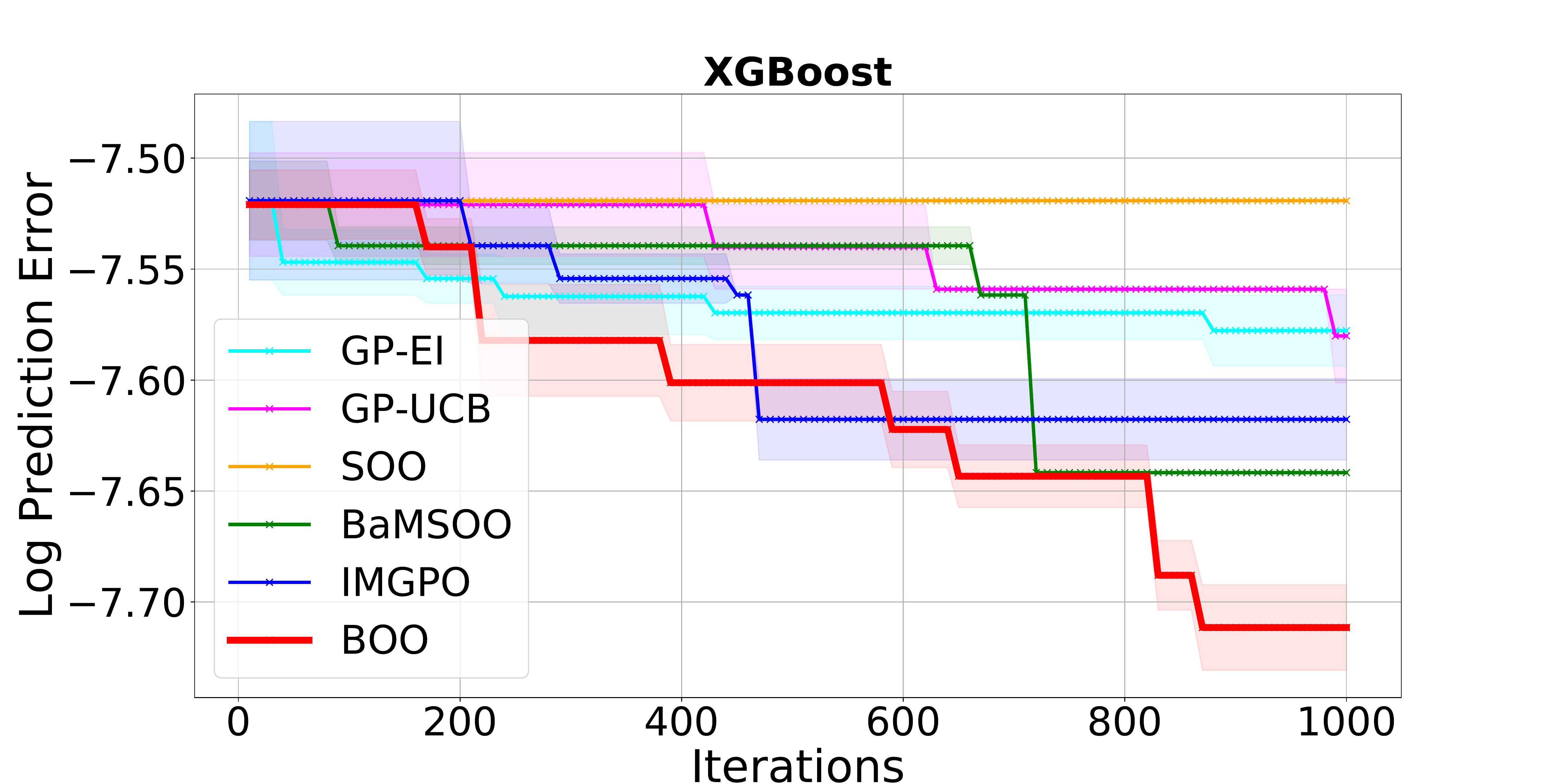}
}
\vspace*{-3mm}
\caption{Log prediction error on MNIST dataset for different algorithms ElasticNet, MLP and XGBoost.}
\label{Log-prediction-error}
\vspace*{-3mm}
\end{figure*}
We first demonstrate the efficiency of our algorithm on standard benchmark functions: Hartmann3 ($D =3$), Schwefel ($D =3$) and Shekel ($D =4$). The evaluation metric is the log distance to the true optimum: $log_{10}(f(x^{*})-f^{+})$, where $f^{+}$ is the best function value sampled so far.

For our BOO algorithm, we choose parameters $m, a, b$ and $N$ as per Corollary 1 which suggests using $N$ so that $a = \mathcal O((\frac{\sqrt{N}}{2})^{1/D}) \ge 2$. For Hartmann3 ($D =3)$ and Schwefel ($D =3$) we use partitioning procedure $P(8;2,3)$ with $N = 200$. For Shekel function ($D=4$), we use $P(16;2;4)$ with $N = 800$ so that $(\frac{\sqrt{N}}{2})^{1/D} \approx 2$. We follow Lemma 6 in our theoretical analysis to set $\beta_{p}=2log(\pi^{2}p^{3}/3\eta)$, with $\eta = 0.05$.

Figure \ref{comparison-of-baseline} shows the performance of our algorithm compared to the baselines.  Our method outperforms all baselines for all considered synthetic functions in general with only one exceptional case of Shekel function where GP-UCB performs better our method. Compared to BaMSOO and IMGPO which are tree-based optimisation algorithms, the efficiency of BOO is gained by using a large $m$ and sampling strategies similar to BO (as shown in Section 4.2). Compared to GP-EI and GP-UCB, our algorithm takes advantage of searching a point to be evaluated at each iteration. BOO searches it only in a promising region (as done in Line 4 and 5 in Algorithm 1) rather in a whole search space. Moreover, unlike GP-EI and GP-UCB, BOO avoids the searching by optimisation at each iteration which cannot be obtained sufficiently and accurately given a limited computation budget.
\paragraph{On Computational Effectiveness} Our method performs competitively against BaMSOO and IMGPO in terms of computational effectiveness (as shown in Table \ref{Table_Time}). Our method uses a large value of $m$ and hence it takes slightly more time to compute UCBs of all nodes. It performs slower than IMGPO but much faster than GP-EI, GP-UCB which require the maximisation of the acquisition function in a continuous space.
\subsection{Hyperparameter Tuning for Machine Learning Models}
To further validate the performance of our algorithm, we tune hyperparameter tuning of three machine learning models on the MNIST dataset and Skin Segmentation dataset, then plot the log prediction error.

\paragraph{Elastic Net} A regression method has the $L_{1}$ and $L_{2}$ regularisation parameters. We tune $w_{1}$ and $w_{2}$ where $w_{1}>0$ expresses the magnitude of the regularisation penalty while $w_{2}\in [0,1]$ expresses the ratio between the two penalties. We tune $w_{1}$ in the normal space while $w_{2}$ is tuned in an exponent space (base 10). The search space is the domain $[0,1]\times[-3,-1]$. We implement the Elastic net model by using the function SGDClassifier in the scikit-learn package.
\paragraph{Multilayer Perceptron (MLP)} We consider a 2-layer MLP with 512 neurons/layer and optimize three hyperparameters: the learning rate $l$ and the $L_{2}$ norm regularisation parameters $l_{r1}$ and $l_{r2}$ of the two layers (all tuned in the exponent space (base 10)). The search space is $[-6,-1]^{3}$. The model is trained with the Adam optimizer in 20 epochs with batch size 128.

Using MNIST dataset, we train the models with this hyperparameter setting using the 55000 patterns and then test the model on the 10000 patterns. The algorithms suggests a new hyperparameter setting based on the prediction accuracy on the test dataset. We set $N =200$. We use $P(4;2,2)$ for ElasticNet and $P(8;2,3)$ for MLP as per Corollary 1.

As seen in Figure \ref{Log-prediction-error}, for Elastic Net, our algorithm outperforms the all baselines. For MLP, our algorithm achieves slightly lower prediction errors compared to the baselines because there is a little room to improve where the prediction error of our method for MLP attains $1.8\%$.
\begin{table}[ht]
\centering
\caption{Hyperparameters for XGBoost.}
\begin{tabular}[t]{lcc}
\hline Variables  & Min & Max \\ \hline
learning rate            & 0.1 & 1  \\
max depth           & 5 & 15  \\
subsample              & 0.5   & 1    \\
colsample           & 0.1  & 1   \\
gamma            & 0  & 10
\end{tabular}
\label{Table_2}
\vspace*{-4mm}
\end{table}%
\paragraph{XGBoost classification} We demonstrate a classification task using XGBoost \citep{Chen16} on a Skin Segmentation dataset \footnote{https://archive.ics.uci.edu/ml/datasets/skin+segmentation}. The Skin Segmentation dataset is plit into $15\%$ for training and $85\%$ for testing for a classification problem. There are 5 hyperarameters for XGBoost which is summarized in Table 2. Our proposed BOO is the best solution, outperforming all the baselines by a wide margin.
\vspace*{-2mm}
\section{Conclusion}
We have presented a first practical algorithm which can achieve an exponential regret bound with tightest order $N^{-\sqrt{N}}$ for Baysian optimisation under the assumption that the objective function is sampled from a Gaussian process with a Mat\'ern kernel with $\nu > 4 +\frac{D}{2}$. Our partitioning procedure and the sampling strategy differ from the existing ones. We have demonstrated the benefits of our algorithm on both synthetic and real world experiments. In the future we plan to extend our work to high dimensions and noisy setting.

\nocite{langley00}
\bibliography{Hung-research}
\bibliographystyle{icml2021}

\appendix
\renewcommand\thesection{\arabic{section}}
\onecolumn

\centerline{ \textbf{\huge Supplementary Material} }

\section{Review of SOO, BaMSOO, IMGPO algorithms}
In the first section of the Supplementary Material, we provide the details of SOO \citep{munos2011} and BamSOO \citep{WangSJF14}. The main difference between our proposed BOO algorithm and these algorithms are in the following blue color lines.
\begin{algorithm}[ht]
\caption{The SOO Algorithm \citep{munos2011}} \label{alg:alg}
\textbf{Input}: Parameter $m$\\
\textbf{Initialisation}: Set $\mathcal T_0 = \{(0,0)\}$ (root node). Set $p = 1$. Sample initial points to build $\mathcal D_{0}$.\\
\begin{algorithmic}[1]
\WHILE{True}
    \STATE Set $v_{max} = - \infty$
    \FOR{$h =0$ to \text{min}(\text{depth}($\mathcal T_p$), $h_{max}(p)$)}
        \STATE \textcolor{blue}{Among all leaves $(h,j)$ of depth $h$, select $(h,i) \in \text{argmax}_{(h,j) \in \mathcal L} f(c_{h,j})$}         \IF{$f(c_{h,i}) \ge v_{max}$}
            \STATE Expand node $(h,i)$ by adding $m$ children $(h+1, i_j)$ to tree $\mathcal T_p$
            \STATE \textcolor{blue}{Evaluate all $m$ functional values $f(c_{h+1,i_j})$, where $(h+1, i_j)$ are children of $(h,i)$}
            \STATE Update $v_{max} = f(c_{h,i})$
            \STATE Update $p = p + 1$
        \ENDIF
    \ENDFOR
\ENDWHILE
\end{algorithmic}
\end{algorithm}
\begin{algorithm}[ht]
\caption{The BaMSOO Algorithm \citep{WangSJF14}} \label{alg:alg}
\textbf{Input}: Parameter $m$\\
\textbf{Initialisation}: Set $g_{0,0} = f(c_{0,0})$, $f^+ = g_{0,0}$, $t=1$, $p = 1$, $\mathcal T_0 = \{(0,0)\}$ (root node). Sample initial points to build $\mathcal D_{0}$.\\
\begin{algorithmic}[1]
\WHILE{True}
    \STATE Set $v_{max} = - \infty$
    \FOR{$h =0$ to \text{min}(\text{depth}($\mathcal T_p$), $h_{max}(p)$)}
        \STATE \textcolor{blue}{Among all leaves $(h,j')$ of depth $h$, select $(h,j) \in \text{argmax}_{(h,j') \in \mathcal L} g(c_{h,j'})$}
        \IF{$g(c_{h,j}) \ge v_{max}$}
            \FOR{$i = 0$ to $k-1$}
                \STATE Update $p = p + 1$
                \IF{$\mathcal U_p(c_{h+1, mj+i}) \ge f^+$}
                    \STATE \textcolor{blue}{Set $g(c_{h+1,mj+i}) = f(c_{h+1, mj+i})$}
                    \STATE Set $t = t+1$
                    \STATE $\mathcal D_t = \{\mathcal D_{t-1}, (c_{h+1, mj+i}, g(c_{h+1, mj+i}))\}$
                \ELSE
                    \STATE Set $g(c_{h+1,mj+i}) = \mathcal L_p(c_{h+1, mj+i})$
                \ENDIF
                \IF{$g(c_{h+1,mj+i}) > f^+$}
                    \STATE Set $f^+ = g(c_{h+1,mj+i})$
                \ENDIF
            \ENDFOR
            \STATE Add the children of $(h,j)$ to $\mathcal T_p$
            \STATE Set $v_{max} = g(c_{h,j})$
        \ENDIF
    \ENDFOR
\ENDWHILE
\end{algorithmic}
\end{algorithm}
As we can see, SOO and BaMSOO select a node to be expanded at line 4 in each algorithm.  At depth $h$, among the leaf nodes, SOO selects the node with the maximum functional value, BaMSOO selects the node with the maximum value of function $g$. The function $g$ is defined at line 9 and line 13 in Algorithm 3. Otherwise, the proposed BOO selects the node with the maximum GP-UCB value.

Once a node is selected to be expanded, SOO needs to sample the function at all $m$ children nodes (at line 7 in Algorithm 2), BamSOO needs to sample the function at $m'$ children nodes (at line 9 in Algorithm 3), where $0 \le m' \le m$ depending on the condition at line 9 in Algorithm 3. In the worst case, $m' =m$, BamSOO spends $m$ evaluations like SOO. Otherwise, our sampling strategy samples the function only at the parent node. As a result, our strategy requires only one function evaluation irrespective of the value of $m$.  IMGPO \citep{Kawaguchi16} is quite similar to BaMSOO except two differences. Frist, IMGPO do not force the tree to a maximum depth of $h_{max}(p)$ like SOO, BamSOO. Second, IMGPO add a strategy to reduce the computation when searching in the tree is inefficient. Please see their paper \citep{Kawaguchi16} for details.
\subsection{Strict Negative Correlation}
As we discussed in section 4.1 of the main paper. Most of tree-based optimistic optimisation algorithms like SOO, StoSOO \citep{ValkoCM13}, BaMSOO and IMGPO face a \emph{strict negative correlation} between the branch factor $m$ and the number of tree expansions given a fixed function evaluation budget $N$. In this part, we provide a summary table showing the simple regret (in the worst case) of these algorithms given a fixed function evaluation budget $N$.
\begin{table}[ht]
\centering
\begin{tabular}[t]{lc}
\hline Algorithm & Simple Regret \\ \hline
SOO              & $\mathcal O(e^{\sqrt{\frac{N}{m}}})$ \\
BaMSOO           & $\mathcal O((\frac{N}{m})^{-\frac{2\alpha}{D(4- \alpha)}})$   \\
IMGPO            & $\mathcal O(e^{\sqrt{\frac{N}{m}}})$  \\ \hline
\end{tabular}
\caption{The simple regret of SOO depends on the near-optimality dimension $d$. If $d > 0$ then the simple regret is sublinear, if $d =0$ then the simple regret is exponential as we show in this table. BaMSOO has a sublinear rate because it uses $d = D/\alpha - D/4$ where $\alpha =1 $ or $2$. IMGPO uses a fixed $m =3$. We here generalize their proof to any $m$.}
\end{table}%

The Table 3 shows the strict negative correlation of tree-based optimistic optimisation algorithms like SOO, BamSOO, IMGPO. The larger $m$ is, the higher the simple regret is. This explains why most of tree-based optimistic optimisation algorithms often use a small value of $m$ like $m =2$, $m =3$. In contrast, our algorithm leverages the large value of $m$ to improve the regret bound.
\section{Proof of Lemma 1}
\begin{lemma}[Lemma 1 in the main paper]
Given any $(a,b)\in M(m)$ and a partitioning procedure $P(m;a,b)$, then
\begin{enumerate}
  \item the longest side of a cell at depth $h$ is at most $a^{-\floor{\frac{bh}{D}}}$, and
  \item the smallest side of a cell at depth $h$ is at least $a^{-\ceil{\frac{bh}{D}}}$.
\end{enumerate}
\label{diameter}
\end{lemma}
\begin{proof}
We prove the statement by induction. At depth $h =1$, we partition the search space $\mathcal X$ into $m = a^b$ cells using the partitioning procedure $P(m;a,b)$. There are two cases on $b$.
\begin{itemize}
  \item $b = D$. Then the longest side of a cell at depth $h =1$ is $1/a = a^{-\floor{\frac{b}{D}}}$. Also, the smallest side of a cell at depth $h=1$ is $1/a = a^{-\ceil{\frac{b}{D}}}$.
  \item $b < D$. Then by the partitioning procedure, the longest side of a cell at depth $h =1$ is still 1. $a^{-\floor{\frac{b}{D}}} = a^0 =1$. Hence, the longest side of a cell at depth 1 is $a^{-\floor{\frac{b}{D}}}$. Also, the smallest side of a cell at depth $h=1$ is $1/a = a^{-\ceil{\frac{b}{D}}}$.
\end{itemize}
For both cases, the statement is true for $h = 1$. We assume that the statement is true for $h \ge 1$. We consider any cell at depth $h + 1$. By our algorithm, this cell is divided from a cell at depth $h$. Similar to the case $h = 1$, we also consider two cases on $b$.
\begin{itemize}
  \item $b =D$. By the inductive hypothesis, the longest side of a cell at depth $h$ is at most $a^{-h}$. Then the longest side of a child cell of this cell is $a^{-(h+1)} = a^{-\floor{\frac{b(h+1)}{D}}}$. Also, the smallest side of a child cell of this cell is $a^{-(h+1)}= a^{-\ceil{\frac{b(h+1)}{D}}}$.
  \item $b < D$. By the inductive hypothesis, the longest side of a cell at depth $h$ is at most $a^{-\floor{\frac{bh}{D}}}$. If we divide a cell at depth $h$ by the partitioning procedure, then the longest side of the sub-cell is at most $a^{-\floor{\frac{bh}{D}}}/a = a^{-1 - \floor{\frac{bh}{D}}}$. However, since $b < D$, $\floor{\frac{b(h + 1)}{D}} \le 1 + \floor{\frac{bh}{D}}$. It follows that $a^{-1 - \floor{\frac{bh}{D}}} \le a^{-\floor{\frac{b(h + 1)}{D}}}$. Thus, the longest side of a cell at depth $h+1$ is at most $a^{-\floor{\frac{b(h+1)}{D}}}$.

      Also, by the inductive hypothesis, the smallest side of a cell at depth $h$ is at least $a^{-\ceil{\frac{bh}{D}}}$. If we divide a cell at depth $h$ then the smallest side of the sub-cell is at least $a^{-\ceil{\frac{bh}{D}}}/a = a^{-\ceil{\frac{bh}{D}}-1}$. However, since $b < D$, $\ceil{\frac{bh}{D}}+1 \ge \ceil{\frac{b(h+1)}{D}}$. As a result, $a^{-\ceil{\frac{bh}{D}}-1} \ge a^{-\ceil{\frac{b(h +1)}{D}}}$. Thus, the smallest side of a cell at depth $h+1$ is at least $a^{-\floor{\frac{b(h+1)}{D}}}$.
\end{itemize}
Thus, the statement holds for every $h \ge 1$.
\end{proof}
\section{Proof of Lemma 4}
To derive an upper bound on variance function $\sigma_p$ as in Lemma 4, we use a concept, called the \emph{fill distance}. Given a set of points $\mathcal{D}_{p-1}$ , we define the fill distance $\text{FD}(\mathcal{D}_{p-1},\mathcal{X})$ as the largest distance from any point in $\mathcal{X}$ to the points in $\mathcal{D}_{p-1}$, as $$\text{FD}(\mathcal{D}_{p-1},\mathcal{X})=\text{sup}_{x\in\mathcal{X}}\text{inf}_{c_{i}\in D_{p-1}}||x-c_{i}||.$$
The following result, which is proven by \citet{Wu92} [Theorem 5.14], after is reviewed by \citet{kanagawa2018} [Theorem 5.4], provides an upper bound for the posterior variance in terms of the fill distance. It applies the cases where the kernel whose RKHS is norm-equivalent to the Sobolev space.
\begin{lemma}[\citep{Wu92,kanagawa2018}]
Let $k$ be a kernel on $\mathbb{R}^d$ whose RKHS is norm equivalent to the Sobolev space. There exist constants $h_0 > 0$ and $C' > 0$ satisfying the following: for any $x \in \mathcal X$ and any set of observations $\mathcal D_{p-1} = \{c_1, c_2,...,c_{p-1}\} \in \mathcal X$ satisfying $\text{FD}(\mathcal{D}_{p-1},\mathcal{X}) \le h_0$, we have
$$\sigma_{p}(x)\le  C'\text{FD}(\mathcal{D}_{p-1},\mathcal{X}))^{\nu - D/2}.$$
\label{le_sobolev}
\end{lemma}
It was shown in \citep{Bull11} [Lemma 3] and in \citep{kanagawa2018} that the Mat\'erm kernels's RKHS is norm-equivalent to the Sobolev space. Therefore, Lemma \ref{le_sobolev} is correct all functions satisfying our Assumption 1 and 2 (in Baysian setting).

Based on Lemma \ref{le_sobolev}, we obtain the following result which is similar to Lemma 4 of \citet{vakili20} but for the Bayesian setting.
\begin{lemma}[Based on Lemma 4 of \citet{vakili20}]
There exist constants $h_0 > 0$ and $C' > 0$ satisfying the following: for any $x \in \mathcal X$ and any set of observations $\mathcal D_{p-1} = \{x_1, x_2,...,x_{p-1}\} \in \mathcal X$ satisfying $\text{FD}(\mathcal{D}_{p-1},\mathcal{X}) \le h_0$, we have
$$ \sigma_p(x) \le \text{min}_{c_i \in \mathcal D_{p-1} }C'||x - c_i||^{\nu - D/2}$$
\label{le_min}
\end{lemma}
\begin{proof}
The proof is very similar to their proof. We include it for the purpose of being self-contained.
For $x \in \mathcal X$, let $c_i \in \mathcal D_{p-1}$ be the closet point to $x$: $||x -c_j|| = \text{min}_{x_i \in \mathcal D_{p-1}}||x -c_i||$. Define $\mathcal X' = \mathcal B_D(c_j, ||x- c_j||)$, the $D$-dimensional hyper-ball centered at $c_j$ with radius $||x -c_j||$. Let $X" = \mathcal D_{p-1}\cap X'$. The fill distance of the points $X"$ in $X'$ satisfies:
$$\text{FD}(X",X') = \text{sup}_{x' \in X'} \text{inf}_{c_i \in X"} ||x' - c_i|| \le \text{sup}_{x' \in X'}||x'-c_j|| = ||x - c_j||.$$
Define $\mu'(x) = \mathbb{E}[f(x)|X"]$ and $k'(x,x') = \mathbb{E}[(f(x)-\mu'(x))(f(x') - \mu'(x'))|X"]$. Let $\sigma'(x) = \sqrt{k'(x, x')}$ be the predictive standard deviation conditioned on observations $X"$. Applying Lemma \ref{le_sobolev} to $\sigma'(x)$, we have
$$\sigma'(x) \le C'\text{FD}(X", X')^{\nu - D/2} \le  C'||x - c_j||^{\nu - D/2}.$$

The lemma holds because $\sigma_p (x) \le \sigma'(x)$. This is the decreasing monotonicity of the variance function. $\sigma'(x)$ is constructed from set $X'' \subset \mathcal D_{p-1}$.  A more formal proof that $\sigma_p (x) \le \sigma'(x)$ can be found in \citep{Chevalier13}.
\end{proof}
Next, we apply this result to our context in which the set of the sampling points $c_i$, $\mathcal D_{p-1} = \{c_1, ..., c_{p-1}\}$ , contains the centers of cells $A_{h,i}$ of a tree structured search space.

Now we prove Lemma 4 in the main paper.
\begin{lemma}
Assuming that node $c_{h,i}$ at the depth $h \ge 1$ was sampled at the $p$-th expansion, where $p \ge h$, then we have that
$$\sigma_{p}(c_{h,i})\le  C_1(\delta(h-1;a,b))^{\nu/2-D/4},$$
where $C_1$ is a constant.
\label{sigma}
\end{lemma}
\begin{proof}
By Lemma \ref{le_min}, for every $x \in \mathcal X$, we have that
$$ \sigma_p(x) \le \text{min}_{c_i \in \mathcal D_{p-1} }C'||x - c_i||^{\nu-D/2}$$
where $C'$ is a constant.

By assumption, node $c_{h,i}$ at depth $h$ is sampled at the $p$-th expansion, where $p \ge h$. By hierarchical structure of the sampled points, node $c_{h,i}$ is sampled only if its parent node was sampled. We denote this node by $c_{h-1,j}$ which is at depth $h-1$ with some index $j$.  It follows that
\begin{eqnarray*}
\sigma_p(c_{h,i}) & \le & \text{min}_{c_i \in \mathcal D_{p-1} }C'||c_{h,i} - c_i||^{\nu-D/2} \\
& \le &  C'||c_{h,i} - c_{h-1, j}||^{\nu-D/2}\\
& \le &  C' (L_1D)^{D/4 - \nu/2}(\delta(h-1;a,b))^{\nu/2-D/4},
\end{eqnarray*}
where in the first inequality, we apply Lemma \ref{le_min}. In the second inequality, we use the property of $c_{h-1,j} \in \mathcal D_{p-1}$, hence $\text{min}_{c_i \in \mathcal D_{p-1}}||c_{h,i} - c_i||^{v- D/2} \le ||c_{h,i} - c_{h-1,j}||^{v- D/2}$. In the last inequality, we have that $c_{h,i}$ belongs to the cell $A_{h-1,j}$ with center $c_{h-1,j}$. Hence, distance $||c_{h,i} - c_{h-1,j}||$ must be shorter than the diameter of that cell. By Lemma \label{diameter} and the definition of $\delta(h-1;a,b) = L_1Da^{-2\floor{\frac{b(h-1)}{D}}}$, the last inequality is proven.

Finally, by setting $C_1 = C'(L_1D)^{D/4 - \nu/2}$, the lemma holds.
\end{proof}
\section{Proof of Theorem 1}
To prove Theorem 1, we will involve two stages:
\begin{itemize}
  \item \textbf{Stage 1}: we first prove that if $N$ is large enough, then under some assumptions, all the centers of nodes of expandable nodes will fall into the ball $\mathcal B(x^*, \theta)$ which is centered at  $x^*$ with radius $\theta$ as defined in Property 1. We prove this in the following Lemma \ref{h0}.
  \item \textbf{Stage 2}: when a set of expandable nodes fallen into the ball $\mathcal B(x^*, \theta)$, the quadratic behaviours of the objective function surrounding the global optimum $x^*$ will occur. We exploit this property to prove that $|I_h| \le C$, where $C$ is some constant.
\end{itemize}

\begin{lemma}
Assume Algorithm 1 uses partitioning procedure $P(m;a,b)$ where $a = \mathcal O(N^{1/D})$ and $b = D$. Thus there exists a constant $N_0$ such that for every $N \ge N_0$, if (1) node $(h,i) \in I_h$, where $h \ge 2$ and (2) $\mathcal L_p(c_{h,i}) \le f(c_{h,i})$ for every $ h \le p \le N$, then $$c_{h,i} \in \mathcal B(x^*, \theta),$$
where $\mathcal B(x^*, \theta)$ is the ball centered at $x^*$ with radius $\theta$, which is defined in Property 1.
\label{h0}
\end{lemma}
\begin{proof}
By definition, the expansion set $I_h = \{(h,i)| \exists h \le p \le N : \mathcal{U}_{p}(c_{h,i})\ge f(x^{*})-\delta(h;a,b) \}$. Therefore, if node $(h,i) \in I_h$ then there must exist some $h \le p \le N$ such that
\begin{eqnarray}
\mathcal U_p(c_{h,i}) \ge f(x^*) - \delta(h;a,b) \label{eq1}
\end{eqnarray}
On the other hand, for the same upper confidence bound $\mathcal U_p(c_{h,i})$ of $f(c_{h,i})$ as above, we have that
\begin{eqnarray}
\mathcal U_{p}(c_{h,i}) & = & \mu_p(c_{h,i}) + \beta^{1/2}_p\sigma_{p}(c_{h,i}) \\
& = & \mu_p(c_{h,i}) - \beta^{1/2}_p\sigma_{p}(c_{h,i}) + 2\beta^{1/2}_p\sigma_{p}(c_{h,i}) \\
& = & \mathcal L_p(c_{h,i}) + 2\beta^{1/2}_p\sigma_{p}(c_{h,i}) \\
& \le & f(c_{h,i}) + 2\beta^{1/2}_p\sigma_{p}(c_{h,i}) \\
& \le & f(c_{h,i}) + 2\beta^{1/2}_pC_1(\delta(p-1;a,b))^{\nu/2 - D/4},
\end{eqnarray}
where in Eq (5), we use the assumption that $\mathcal L_p(c_{h,i}) \le f(c_{h,i})$. In Eq (6), we use Lemma \ref{sigma}.

Combining Eq (1) and Eq (7), we obtain
\begin{eqnarray}
f(x^*) - f(c_{h,i}) & \le & f(c_{h,i}) + 2\beta^{1/2}_pC_1(\delta(p-1;a,b))^{\nu/2 - D/4} \\
& = & L_1Da^{-2\floor{\frac{bh}{D}}} + 2\beta^{1/2}_pC_1(L_1Da^{-2\floor{\frac{b(h-1)}{D}}})^{\nu/2 - D/4} \\
& = & L_1Da^{-2h} + 2\beta^{1/2}_p C_1 (L_1D)^{\nu/2 - D/4} a^{-(\nu - D/2)(h-1)},
\end{eqnarray}
where Eq (8) uses the definition of $\delta(h;a,b)$ and Eq (9) uses the assumption that $b = D$.

We continue to go further with Eq (9) by using the assumptions that $a = \mathcal O(N^{1/D})$, $h \ge 2$ (from assumptions of Lemma \ref{h0}), and $\nu -D/2> 4$ (from Assumption 1):
\begin{eqnarray}
f(x^*) - f(c_{h,i}) &\le & L_1Da^{-2h} + 2\beta^{1/2}_p C_1 (L_1D)^{\nu/2 - D/4} a^{-(\nu - D/2)(h-1)} \\
& \le & L_1Da^{-4} + 2\beta^{1/2}_N C_1 (L_1D)^{\nu/2 - D/4}a^{-4}\\
& = & \frac{L_1D + 2C_1 (L_1D)^{\nu/2 - D/4}\sqrt{2log(\pi^2N^3/3\eta)}}{a^4}  \\
& = & \mathcal O (\frac{\sqrt{log(N/3\eta)}}{N^{4/D}}),
\end{eqnarray}
where, in Eq (11), we use $h \ge 2$ and the increasing monotonicity of function $\beta_p$. We recall that $\beta_p$ is the trade-off parameter used on our BOO proposed. Formally, $\beta_p = 2log(\frac{\pi^2p^3}{3\eta})$, where $\eta \in (0, 1)$. In Eq (13), we use $a = \mathcal O(N^{1/d})$.

We have that $\frac{\sqrt{log(N/3\eta)}}{N^{4/D}} \rightarrow 0$ as $N \rightarrow \infty$. Therefore, for any $\epsilon_0 > 0$, there exists a constant $N_0 > 0$ such that for every $ N \ge N_0$, $f(x^*) - f(c_{h,i}) \le \epsilon_0$.  Thus, by definition of $\mathcal B(x^*, \theta)$ in Property 1, $c_{h,i} \in \mathcal B(x^*, \theta)$.
\end{proof}
We now start to prove Theorem 1.
\begin{theorem}
Assume that the proposed BOO algorithm uses partitioning procedure $P(m;a,b)$ where $a = \mathcal O(N^{1/D})$ and $b =D$. We consider set $I_h$, where $h \ge 2$ and assume that $\mathcal L_p(c_{h,i}) \le f(c_{h,i}) \le \mathcal U_p(c_{h,i})$ for all node $(h,i) \in I_h $ and for all $h \le p \le N$. Then there exist constants $N_1 > 0$ and $C > 0$ such that for every for $N \ge N_1$,
$$|I_h| \le C.$$
\label{bound_size}
\end{theorem}

\begin{proof}
The proof involves three steps.

\textbf{Step 1}: \textbf{for each node $(h,i) \in I_h$, we seek to bound gap $||x^* - c_{h,i}||$}.

By Lemma \ref{h0} and the assumptions of Lemma \ref{bound_size}, there exists a constant $N_0$ such that for every $N \ge N_0$, for any $(h,i) \in I_h$ then $c_{h,i} \in \mathcal B(x^*, \theta)$. Hence following Property 1, for any $(h,i) \in I_{h}$, the following result is guaranteed:
\begin{eqnarray}
L_2||x^* - c_{h,i}||^2 \le f(x^*) - f(c_{h,i}).
\end{eqnarray}
On the other hand, by definition of $I_h$, there exists $h \le p \le N$ such that
\begin{eqnarray}
\mathcal U_p(c_{h,i}) \ge f(x^*) - \delta(h;a,b).
\end{eqnarray}
Combining Eq (14) and Eq (15), we have that
\begin{eqnarray}
L_2||x^* - c_{h,i}||^2 \le \mathcal U_p(c_{h,i}) + \delta(h;a,b) - f(c_{h,i}).
\end{eqnarray}
Similar to Lemma \ref{h0}, we continue to analyze the right hand side of Eq (16) as follows:
\begin{eqnarray}
L_2||x^* - c_{h,i}||^2 & \le & \mathcal U_p(c_{h,i}) + \delta(h;a,b) - f(c_{h,i}) \\
& =  & \mu_p(c_{h,i}) + \beta^{1/2}_p\sigma_{p}(c_{h,i}) + \delta(h;a,b) - f(c_{h,i}) \\
& =  & \mu_p(c_{h,i}) - \beta^{1/2}_p\sigma_{p}(c_{h,i}) + 2\beta_p^{1/2}\sigma_{p}(c_{h,i}) +  \delta(h;a,b) - f(c_{h,i})\\
& \le & \mathcal L_p(c_{h,i}) + 2\beta^{1/2}_p\sigma_{p}(c_{h,i}) +  \delta(h;a,b) - f(c_{h,i}) \\
& \le & 2\beta^{1/2}_p\sigma_{p}(c_{h,i}) +  \delta(h;a,b) \\
& \le & 2\beta^{1/2}_p C_1(\delta(p-1;a,b))^{\nu/2 -D/4} +  \delta(h;a,b) \\
& \le & 2\beta^{1/2}_N C_1(\delta(h-1;a,b))^{\nu/2 -D/4} +  \delta(h;a,b),
\end{eqnarray}
where in Eq (18), we use the definition of $\mathcal U_p(c_{h,i})$, in Eq (20), we use the definition of $\mathcal L_p(c_{h,i})$. In Eq (21), we use the assumption that $\mathcal L_p(c_{h,i}) \le f(c_{h,i})$. In Eq (22), we use Lemma \ref{sigma}. Finally, in the last inequality at Eq (23), we use the decreasing monotonicity of function $\delta(h;a,b)$ and the increasing monotonicity of function $\beta_p$. By assumption that $h \le p \le N$, hence $\delta(p-1;a,b) \le \delta(h-1;a,b)$ and $\beta^{1/2}_p \le \beta^{1/2}_N$. We recall that $\delta(h;a,b) = L_1Da^{-2\floor{\frac{bh}{D}}}$ as in Definition 1.

Thus, for any $(h,i) \in I_h$, where $h \ge 2$, we have that
$$L_2||x^* - c_{h,i}||^2 \le 2\beta^{1/2}_N C_1(\delta(h-1;a,b))^{\nu/2 -D/4} +  \delta(h;a,b).$$

\textbf{Step 2}: \textbf{Bounding $|I_h|$ using covering balls}.

We let $\Omega_h$ be the set of nodes $(h,i)$ at depth $h$ generated by partitioning procedure $P(m;a,b)$. From $\Omega_h$ we define set $\overline{I}_h$ as $$\overline{I}_h = \{(h,i) \in \Omega_h \text{ such that } L_2||x^* - c_{h,i}||^2 \le 2\beta^{1/2}_N C_1(\delta(h-1;a,b))^{\nu/2 -D/4} +  \delta(h;a,b) \}.$$

By this definition, $I_h \subseteq \overline{I}_h$ which implies directly that $|I|_h \le |\overline{I}_h|$. Now we consider the set of points $c_{h,i}$ of these nodes. This set is defined as $$\overline{P}_h = \{c_{h,i} \in \mathcal X | (h,i)\in  \overline{I}_h \}.$$

We can see that all the points of $\overline{P}_h$ are covered by a hypersphere centered at $x^*$  with radius $ \sqrt{\frac{2\beta^{1/2}_N C_1(\delta(h-1;a,b))^{\nu/2 -D/4} +  \delta(h;a,b)}{L_2}}$. We call this hypersphere $\mathcal S_h$.

On the other hand, by Lemma \ref{diameter}, the smallest side of a cell $A_{h,i}$ at depth $h$ is at least $a^{\ceil{-\frac{bh}{D}}}$. Therefore, if we bound a point $c_{h,i} \in \overline{I}_h$ by a $D$-ball centered $c_{h,i}$ with radius $a^{\ceil{-\frac{bh}{D}}}/2$ then all these balls are disjoint. Further, even if there are several centers $c_{h,i}$ of these balls lying on the boundary of $\mathcal S_h$ then all these balls must be within the hypersphere centered at $x^*$  with radius $$ \sqrt{\frac{2\beta^{1/2}_N C_1(\delta(h-1;a,b))^{\nu/2 -D/4} +  \delta(h;a,b)}{L_2}} +  a^{\ceil{-\frac{bh}{D}}}/2.$$

Thus, $|\overline{P}_h|$ cannot exceed the number of disjoint balls which fit in the hypersphere centered at $x^*$  with radius $ \sqrt{\frac{2\beta^{1/2}_N C_1(\delta(h-1;a,b))^{\nu/2 -D/4} +  \delta(h;a,b)}{L_2}} +  a^{\ceil{-\frac{bh}{D}}}/2$.

The number of these disjoint balls cannot exceed the proportion of the volume of the hypersphere of radius $\sqrt{\frac{\delta(h) + 2\beta^{1/2}_NC_2L^{-\nu + D/2}(\delta(h-1))^{\nu - D/2}}{C_1}}$ and the volume of small balls of radius $a^{\ceil{-\frac{bh}{D}}}/2$.
This proportion is measured by
$$(\frac{\sqrt{\frac{2\beta^{1/2}_N C_1(\delta(h-1;a,b))^{\nu/2 -D/4} +  \delta(h;a,b)}{L_2}} +  a^{\ceil{-\frac{bh}{D}}}/2}{a^{\ceil{-\frac{bh}{D}}}/2})^D.$$

Thus, we have that
\begin{eqnarray}
|\overline{P}_h| & \le & (\frac{\sqrt{\frac{2\beta^{1/2}_N C_1(\delta(h-1;a,b))^{\nu/2 -D/4} +  \delta(h;a,b)}{L_2}} +  a^{\ceil{-\frac{bh}{D}}}/2}{a^{\ceil{-\frac{bh}{D}}}/2})^D \\
& = & (\sqrt{\frac{4\beta^{1/2}_N C_1(\delta(h-1;a,b))^{\nu/2 -D/4} +  2\delta(h;a,b)}{L_2 a^{2\ceil{-\frac{bh}{D}}}}} + 1)^D
\end{eqnarray}
However, by definition of $\overline{I}_h$ and $\overline{P}_h$, $|I_h| \le |\overline{I}_h| = |\overline{P}_h|$. Therefore, we have
$$|I_h| \le (\sqrt{\frac{4\beta^{1/2}_N C_1(\delta(h-1;a,b))^{\nu/2 -D/4} +  2\delta(h;a,b)}{L_2 a^{2\ceil{-\frac{bh}{D}}}}} + 1)^D.$$

\textbf{Step 3:} \textbf{proving that there exists a constant $C$ such that $|I_h| \le C$}.

Using the assumption that $b =D$, we have $a^{\ceil{-2\frac{bh}{D}}} = a^{-2h}$, $\delta(h-1;a,b) = L_1Da^{-2(h-1)}$, and $\delta(h;a,b) = L_1Da^{-2h}$. Replacing these results to Eq (25), we get
\begin{eqnarray}
|I_h| & \le & (\sqrt{\frac{4\beta^{1/2}_N C_1(\delta(h-1;a,b))^{\nu/2 -D/4} +  2\delta(h;a,b)}{L_2 a^{2\ceil{-\frac{bh}{D}}}}} + 1)^D \\
& = & (\sqrt{\frac{4\beta^{1/2}_N C_1(L_1D)^{\nu/2 -D/4} a^{-(\nu -D/2)(h-1)} +  2L_1Da^{-2h}}{L_2 a^{-2h}}} + 1)^D \\
& = & (\sqrt{ \frac{4C_1(L_1D)^{\nu/2 -D/4}}{L_2} \times \beta^{1/2}_N  \times a^{2h-(\nu -D/2)(h-1)} + \frac{2L_1}{L_2} } + 1)^D \\
& \le & (\sqrt{ \frac{4C_1(L_1D)^{\nu/2 -D/4}}{L_2} \times \beta^{1/2}_N  \times a^{4 + D/2 -\nu} + \frac{2L_1}{L_2} } + 1)^D \\
& = & C'(\sqrt{log(N/3\eta)}  \times N^{(4 + D/2 -\nu)/D})^{D/2}
\end{eqnarray}
where, Eq (29) holds because $a^{2h-(\nu -D/2)(h-1)} \le a^{4 + D/2 -\nu}$. Indeed, by using the assumption that $\nu > 4 + D/2$ and $h \ge 2$, we have that
\begin{eqnarray*}
a^{2h-(\nu -D/2)(h-1)} & = & a^{h(2 + D/2 -\nu) + (\nu -D/2)} \\
& \le &  a^{2(2 + D/2 -\nu) + (\nu -D/2)} \\
& = & a^{4 + D/2 -\nu}.
\end{eqnarray*}
For the last inequality at Eq (30), we use the assumption $a = \mathcal O(N^{1/D})$, $\beta_N = 2log(\pi^{2}N^{3}/3\eta)$, where $\eta \in (0,1)$, and the fact that $\frac{4C_1(L_1D)^{\nu/2 -D/4}}{L_2}$ and $\frac{2L_1}{L_2}$ are constants independent of $N$. Thus, such a constant $C'$ at Eq (30) exists.

Since $\nu > 4 + D/2$, we have that $\frac{\sqrt{log(N/3\eta)}}{N^{\frac{\nu -D/2 -4}{D}}} \rightarrow 0$ as $N \rightarrow \infty$.
Therefore $\sqrt{log(N/3\eta)} \times N^{(4 + D/2 -\nu)/D} \rightarrow 0$ as $N \rightarrow \infty$. Thus, there exists constant $N_1 > 0$ and $C > 0$ such that for every $N \ge N_1$, $|I_h| \le C$ for every $h \ge 2$.
\end{proof}
\section{Proof of Lemma 5}
Let $(h_{p}^{*}+1,i^{*})$ be an optimal node of depth $h_{p}^{*}+1$ (i.e., $x^{*}\in A_{h_{p}^{*}+1,i^{*}}$). We define a node $(h,i)$ at depth $h$ as $\delta(h;a,b)$-optimal if $\mathcal U(c_{h,i}) \ge f(c_{h,i}) - \delta(h;a,b)$. We obtains the following result.
\begin{lemma}
Assume that $f(c_{h_{p}^{*}+1,i^{*}})\le\mathcal{U}(c_{h_{p}^{*}+1,i^{*}})$. Then any node $(h_{p}^{*}+1,i)$ of depth $h_{p}^{*}+1$ before $(h_{p}^{*}+1,i^{*})$ is expanded, is $\delta(h_{p}^{*}+1;a,b)$-optimal.
\end{lemma}
\begin{proof}
If the node $(h_{p}^{*}+1,i^{*})$ has not been expanded yet, then by Algorithm 1 (line 4) we have that $\mathcal{U}(c_{h_{p}^{*}+1,i})\ge \mathcal{U}(c_{h_{p}^{*}+1,i^{*}})$. Combining with the assumptions, we get
\begin{eqnarray}
\mathcal{U}(c_{h_{p}^{*}+1,i}) & \ge & \mathcal{U}(c_{h_{p}^{*}+1,i^{*}}) \\
& \ge & f(c_{h_{p}^{*}+1,i^{*}}) \\
& \ge & f^{*}-\delta(h_{p}^{*}+1;a,b),
\end{eqnarray}
where Eq (32) use the assumption that $f(c_{h_{p}^{*}+1,i^{*}})\le\mathcal{U}(c_{h_{p}^{*}+1,i^{*}})$, and Eq (33) use Lemma 3. Thus, the lemma holds.
\end{proof}
From Lemma 12, we deduce that once an optimal node of depth $h$ is expanded, it takes at most $|I_{h+1}|$ node expansions at depth $h+1$ before the optimal node of depth $h+1$ is expanded. From that observation, we deduce the following lemma (corresponding to Lemma 5 in the main paper.)
\begin{lemma}
Assume that $f(c_{h,i}) \le \mathcal U(c_{h,i})$ for all optimal node $(h,i)$ at each depth $0 \le h \le h_{max}(n)$. Then for any depth $0 \le h \le h_{max}(n)$, whenever $n \ge h_{max}(n)\sum_{i=0}^{h}|I_i|$, we have $h^*_n \ge h$.
\end{lemma}
\begin{proof}
We prove it by induction. For $h = 0$, we have $h_n^* \ge 0$.

Assume that the proposition is true for all $o \le h \le h_0$ with $h_0 < h_{max}(n)$. Let us prove that it is also true for $h_0 +1$. Let $n \ge h_{max}(n) (|I_0| + |I_1| +...+ |I_{h_0 +1}|)$. Since $n \ge h_{max}(n) (|I_0| + |I_1| +...+ |I_{h_0}|)$, we have $h^*_n \ge h_0$. If $h^*_n \ge h_0 +1$ then the proof is finished. If $h^*_n = h_0$, we consider the nodes of depth $h_0 +1$ that are expanded. We have seen that as long as the optimal node of depth $h_0 +1$ is not expanded, any node of depth $h_0 +1$ that is expanded must be $\delta(h_0+1;a,b)$-optimal, i.e., belongs to $I_{h_0 +1}$. Since there are $|I_{h_0 +1}|$ of them, after $h_{max}(n)|I_{h_0 +1}|$ node expansions, the optimal one must be expanded, thus $h^*_n \ge h_0 +1$.
\end{proof}
\section{Proof of Lemma 6}
We use $A_{p}$ to denote the set of all points evaluated by the algorithm and all centers of optimal nodes of the tree $\mathcal{T}_{p}$ after $p$ evaluations.
\begin{lemma}
Pick a $\eta\in(0,1)$. Set $\beta_{p}=2log(\pi^{2}p^{3}/3\eta)$ and $\mathcal L_p(c) = \mu_p(c) - \beta^{1/2}_p\sigma_p(c)$. With probability $1-\eta$, we have
$$\mathcal L_p(c) \le f(c) \le \mathcal U_p(c),$$
for every $p \ge 1$ and for every $c\in A_{p}$.
\end{lemma}
\begin{proof}
After $p$ evaluations, there are at most $p$ evaluated points by the algorithm. On the other hand, after $p$ evaluations, the deepest depth of the tree $\mathcal T_p$ is $p$. In addition, at each depth, there is only one optimal node which contains $x^*$. Therefore, there are at most $p$ centers of optimal nodes which belong to tree $\mathcal{T}_{p}$. Thus, $|A_p| \le 2p$.

The proof is similar to Lemma 5.1 in \cite{Srinivas12} and Lemma 4 in \cite{WangSJF14} with the set $A_p$ (here we use the fact that $f$ is a sample from the GP). If we let $\beta_{p}=2log(\pi^{2}p^2|A_p|/6\eta)$, then with probability $1 - \eta$,  we have
$$\mathcal L_p(c) \le f(c) \le \mathcal U_p(c),$$
for every $p \ge 1$ and for every $c\in A_{p}$.
Since $|A_p| \le 2p$, we will use $\beta_{p}=2log(\pi^{2}p^3/3\eta)$ instead and the lemma also holds with this $\beta_p = 2log(\pi^{2}p^3/3\eta)$.
\end{proof}
Lemma 6 implies that with probability $1 -\eta$, all conditions $\mathcal L_p(c) \le f(c) \le \mathcal U_p(c)$ in Lemma 4, Theorem 1, and Lemma 5 in the main paper hold for every $1 \le p \le N$.
\section{Proof of Theorem 2}
\begin{theorem}[Regret Bound]
Assume that there is a partitioning procedure $P(m;a,b)$ where $a = \mathcal{O}(N^{1/D})$, $b=D$ and $2 \le m  < \sqrt{N}-1$. Let the depth function $h_{max}(p) = \sqrt{p}$. We consider $m^2 < p \le N$, and define $h(p)$ as the smallest integer $h$ such that $$ h \ge \frac{\sqrt{p}-m-1}{C} + 2,$$
where $C$ is the constant defined by Theorem 1. Pick a $\eta \in(0,1)$. Then for every $N \ge N_1$, the loss is bounded as
\begin{eqnarray*}
r_{p} & \le & \delta(\text{min}\{h(p), \sqrt{p}+1\};a,b)  +  4C_1\beta^{1/2}_p(\delta(\text{min}\{h(p)-1, \sqrt{p}\};a,b))^{\nu/2-D/4},
\end{eqnarray*}
with probability $1-\eta$, where $N_1$ is the constant defined in Theorem 1, $C_1$ is the constant defined in lemma 4 and $\beta_N = \sqrt{2log(\pi^{2}N^{3}/3\eta)}$.
\end{theorem}
\begin{proof}
By Theorem 1, the definition of $h(p)$ and the facts that $|I_0| = 1$ and $|I_1| \le m$, we have
\begin{eqnarray*}
\sum_{l=0}^{h(p)-1} |I_l| & = & |I_0| + |I_1| + (|I_2| + ... + |I|_{h(p)-1}) \\
& \le & 1 + m + C(h(p)-2) \le \sqrt{p}
\end{eqnarray*}
Therefore, $\sum_{l=0}^{h(p)-1}|I_{l}|\le \sqrt{p}$. By Lemma 5 when $h(p)-1\le h_{max}(p)= \sqrt{p}$, we have $h_{p}^{*}\ge h(p)-1$. If  $h(p)-1 > \sqrt{p}$ then $h_{p}^{*}= h_{max}(p) = \sqrt{p}$ since the BOO algorithm does not expand nodes beyond depth $h_{max}(p)$. Thus, in all cases, $h_{p}^{*}\ge\text{min}\{h(p)-1, \sqrt{p}\}$.

Let $(h,j)$ be the deepest node in $\mathcal{T}_{p}$ that has been expanded by the algorithm up to $p$ expansions. Thus $h\ge h_{p}^{*}$. By Algorithm 1, we only expand a node when its GP-UCB value is larger than $v_{max}$ which is updated at Line 10 of Algorithm 1. Thus, since the node $(h,j)$ has been expanded, its GP-UCB value is at least as high as that of the some node $(h_{p}^{*}+1, o)$ at depth $h_{p}^{*}+1$, such that
\begin{itemize}
  \item (1) node $(h_{p}^{*}+1, o)$ has been evaluated at some $p'$-th expansion before node $(h,j)$  and
  \item (2)  $(h_{p}^{*}+1, o) \in  \text{argmax}_{(h_p^{*} +1,i) \in \mathcal L} \mathcal U_{p'}(c_{h_p^{*} +1,i})$ (see Line 4 of Algorithm 1).
\end{itemize}
We let node $(h_{p}^{*}+1, o^*)$ be the optimal node at depth $h_{p}^{*}+1$. With probability $ 1 - \eta$,
\begin{eqnarray}
f(x^*) - \delta(h_{p}^{*}+1; a, b) & \le & f(c_{h_{p}^{*}+1, o^*}) \\
& \le & \mathcal U_{p'}(c_{h_{p}^{*}+1, o^*}) \\
& \le & \mathcal U_{p'}(c_{h_{p}^{*}+1, o}) \\
& \le & \mu_{p'}(c_{h_{p}^{*}+1, o}) +  \beta^{1/2}_{p'}\sigma_{p'}(c_{h_p^{*} +1,o})\\
& \le & \mu_{p'}(c_{h_{p}^{*}+1, o}) -   \beta^{1/2}_{p'}\sigma_{p'}(c_{h_p^{*} +1,o}) + 2\beta^{1/2}_{p'}\sigma_{p'}(c_{h_p^{*} +1,o})\\
& \le & \mathcal L_{p'} (c_{h_{p}^{*}+1, o})  +  2\beta^{1/2}_{p'}\sigma_{p'}(c_{h_p^{*} +1,o})\\
& \le & f(c_{h_{p}^{*}+1, o})  +  2\beta^{1/2}_{p'}\sigma_{p'}(c_{h_p^{*} +1,o})\\
& \le & \mathcal U_p(c_{h, j})  +  2\beta^{1/2}_{p'}\sigma_{p'}(c_{h_p^{*} +1,o}),
\end{eqnarray}
where in Eq (34), we use Lemma 3. Eq (35) holds with probability $1 -\eta$  by using Lemma 6. In Eq (36), we use the above condition (2). Eq (37) uses the definition of $\mathcal U_{p'}$. Eq (39) uses the definition of $\mathcal L_{p'}$. Eq (40) holds with probability $1 -\eta$  by using Lemma 6. Finally, Eq (41) uses the updating condition at Line 5 and Line 10 of Algorithm 1.

Eq (41) implies that with probability $1 -\eta$,
$$f(x^*) - \mathcal U_p(c_{h, j}) \le \delta(h_{p}^{*}+1; a, b) + 2\beta^{1/2}_{p'}\sigma_{p'}(c_{h_p^{*} +1,o}).$$

On the other hand,  by Lemma 6, with probability $1 -\eta$, we have
\begin{eqnarray*}
U_p(c_{h,j}) & = &  \mu_p(c_{h,j}) + \beta^{1/2}_p\sigma_{p}(c_{h,j}) \\
& = &  \mathcal L_p(c_{h,j}) + 2\beta_p^{1/2}\sigma_{p}(c_{h,j}) \\
& \le & f(c_{h,j}) + 2\beta_p^{1/2}\sigma_{p}(c_{h,j})
\end{eqnarray*}

Combining these two results, we have $$f(x^*) - f(c_{h,j}) \le \delta(h_p^{*}+1;a,b) + 2\beta^{1/2}_{p'}\sigma_{p'}(c_{h_p^{*} +1,o}) + 2\beta_p^{1/2}\sigma_{p}(c_{h,j}),$$ with a probability $1 -\eta$.

Finally, by using Lemma 4 to bound $\sigma_{p'}(c_{h_p^{*} +1,o})$ and $\sigma_{p}(c_{h,j})$ and using the fact that the function $\delta(*;a,b)$ decreases with their depths, we achieve
\begin{eqnarray*}
r_p & \le & f(x^*) - f(c_{h,j}) \\
& \le  & \delta(\text{min}\{h(p), \sqrt{p}+1\};a,b) + 4C_1\beta^{1/2}_p(\delta(\text{min}\{h(p)-1, \sqrt{p}\};a,b))^{\nu/2-D/4}
\end{eqnarray*}
with a probability $1 -\eta$.
\end{proof}
\section{Proof of Corollary 1}
\begin{corollary}
Pick a $\eta \in(0,1)$. There exists a constant $N_2 > 0$ such that for every $N \ge N_2$ we have that the simple regret of the proposed BOO with the partitioning procedure $P(m;a,b)$ where $a=\floor{(\frac{\sqrt{N}}{2})^{\frac{1}{D}}}$, $b=D$, is bounded as
$$r_N \le \mathcal O(N^{-\sqrt{N}}),$$ with probability $1-\eta$.
\end{corollary}
\begin{proof}
With $a=\floor{(\frac{\sqrt{N}}{2})^{\frac{1}{D}}}$ and $b=D$, $m = a^b \le \sqrt{N}/2$. These conditions satisfy the assumptions of Theorem 2, therefore following Theorem 2 with probability $1 -\eta$, we have that
\begin{eqnarray*}
r_N  \le  \underbrace{\delta(\text{min}\{h(N), \sqrt{N}+1\};a,b)}_{\text{Term 1}} + \underbrace{4C_1\beta^{1/2}_N(\delta(\text{min}\{h(N)-1, \sqrt{N}\};a,b))^{\nu/2-D/4}}_{\text{Term 2}}.
\end{eqnarray*}
We consider Term 1. There are two cases:

(1) If $\text{min}\{h(N), \sqrt{N}+1\} = \sqrt{N} +1$ then $\delta(\text{min}\{h(N), \sqrt{N}+1\};a,b) = \delta(\sqrt{N}+1;a,b) = L_1D a^{-2(\sqrt{N} +1)} \le \mathcal O(N^{-\sqrt{N}})$ by replacing $a = \floor{(\frac{\sqrt{N}}{2})^{\frac{1}{D}}}$.

(2) If $\text{min}\{h(N), \sqrt{N}+1\} = h(N)$. By definition of $h(N)$ in Theorem 2, $h(N) \ge \frac{\sqrt{N}-m-1}{C} + 2 \ge \frac{\sqrt{N}}{2C} - \frac{1}{C} + 2$. Therefore, $\delta(\text{min}\{h(N), \sqrt{N}+1\};a,b) = \delta(h(N);a,b) = L_1D a^{-2h(N)} \le \mathcal O(N^{-\sqrt{N}})$.

Thus, for both cases, Term 1 is bounded by $O(N^{-\sqrt{N}})$. We now consider Term 2. There are also two cases:

(1) If  $\text{min}\{h(N) - 1, \sqrt{N}\} = \sqrt{N}$ then $4C_1\beta^{1/2}_N(\delta(\text{min}\{h(N)-1, \sqrt{N}\};a,b))^{\nu/2-D/4} = 4C_1\beta^{1/2}_N(\delta(\sqrt{N};a,b))^{\nu/2-D/4} = 4C_1L_1D\beta^{1/2}_N a^{-2(\nu/2-D/4)\sqrt{N}} \le 4C_1L_1D\beta^{1/2}_N a^{-4\sqrt{N}}$. In the last inequality, we use the assumption that $\nu > 4 + D/2$. The component $a^{-4\sqrt{N}}$ with $a = \floor{(\frac{\sqrt{N}}{2})^{\frac{1}{D}}}$ dominates $\beta_N$ which is $\mathcal O(\sqrt{N})$. Therefore Term 2 is bounded by $O(N^{-\sqrt{N}})$.

(2) If  $\text{min}\{h(N) - 1, \sqrt{N}\} = h(N) - 1$. By definition of $h(N)$ in Theorem 2, $h(N) -1 \ge \frac{\sqrt{N}-m-1}{C} + 1 \ge \frac{\sqrt{N}}{2C} - \frac{1}{C} + 1$. Then $4C_1\beta^{1/2}_N(\delta(\text{min}\{h(N)-1, \sqrt{N}\};a,b))^{\nu/2-D/4} = 4C_1\beta^{1/2}_N(\delta(h(N) -1;a,b))^{\nu/2-D/4} = 4C_1L_1D\beta^{1/2}_N a^{-2(\nu/2-D/4)(h(N) -1)} \le 4C_1L_1D\beta^{1/2}_N a^{-4(h(N) -1)}$. By the argument similar as above, we have that Term 2 is bounded by $O(N^{-\sqrt{N}})$.

Finally, for all cases, we get that $r_N \le O(N^{-\sqrt{N}})$ with probability $1 -\eta$.
\end{proof}

\section{Ablation study between function sampling and partitioning procedure in the proposed BOO algorithm}
To show this point, we have performed additional experiments with $m=64$ (see Figure below). The left plot shows different partitioning procedures  while function sampling is fixed to our proposed scheme. We can see good improvement when $b=D$ compared to $b=1$ case. The right plot compares BaMSOO with our method which uses the proposed function sampling scheme but keeps using BaMSOO's partitioning procedure ($b=1$, $m=64$). In this case, we are not able to outperform BaMSOO. However, our result for $b=D$ in the left plot is significantly better than that of BaMSOO. This clearly shows that the effect of partitioning procedure is higher than that of the function sampling.

\begin{figure}[ht]
\centering
\subfigure{\includegraphics[scale=1.0,width=.45\textwidth]{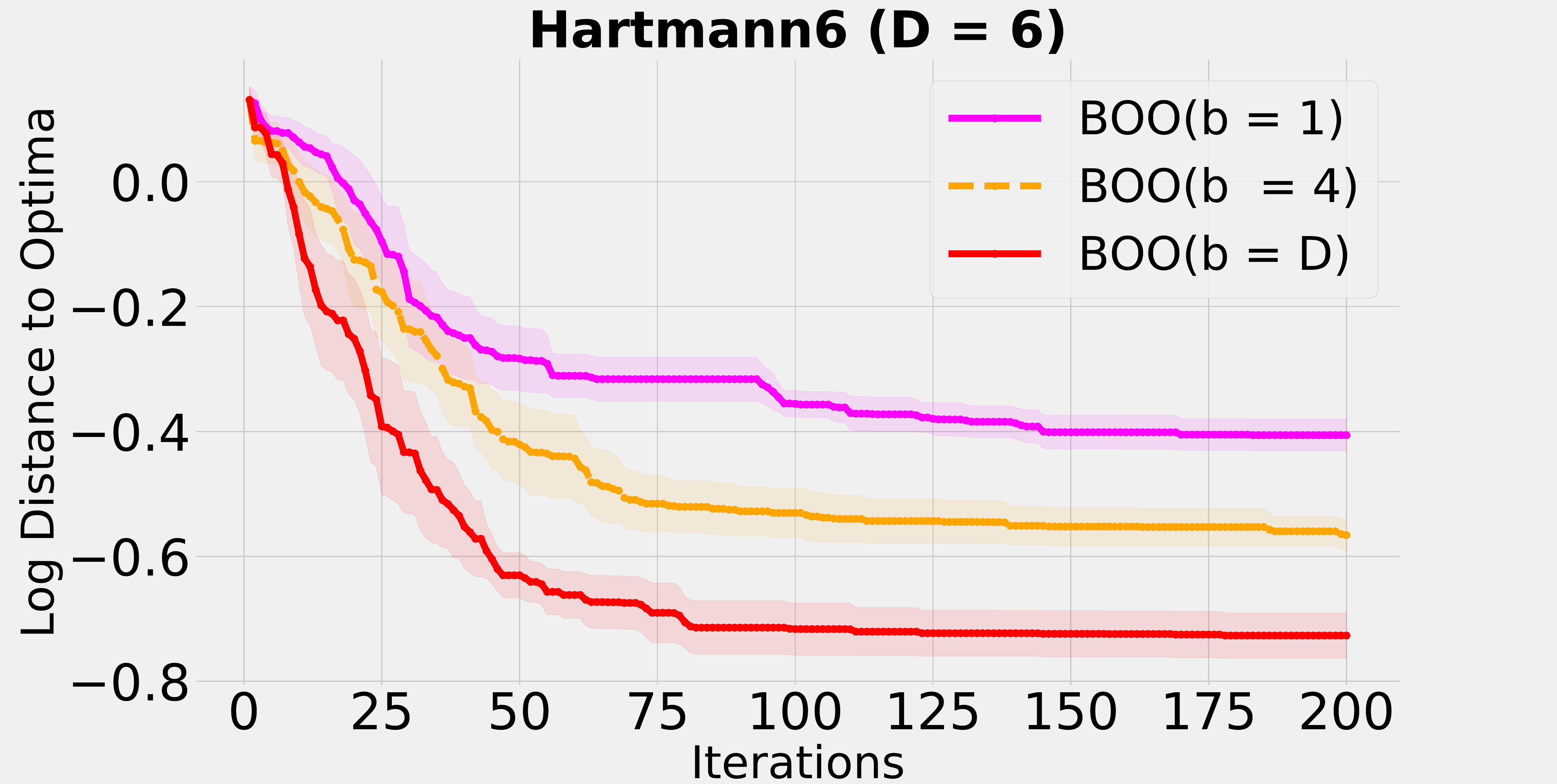}
}\hfill
\subfigure{\includegraphics[scale=1.0,width=.45\textwidth]{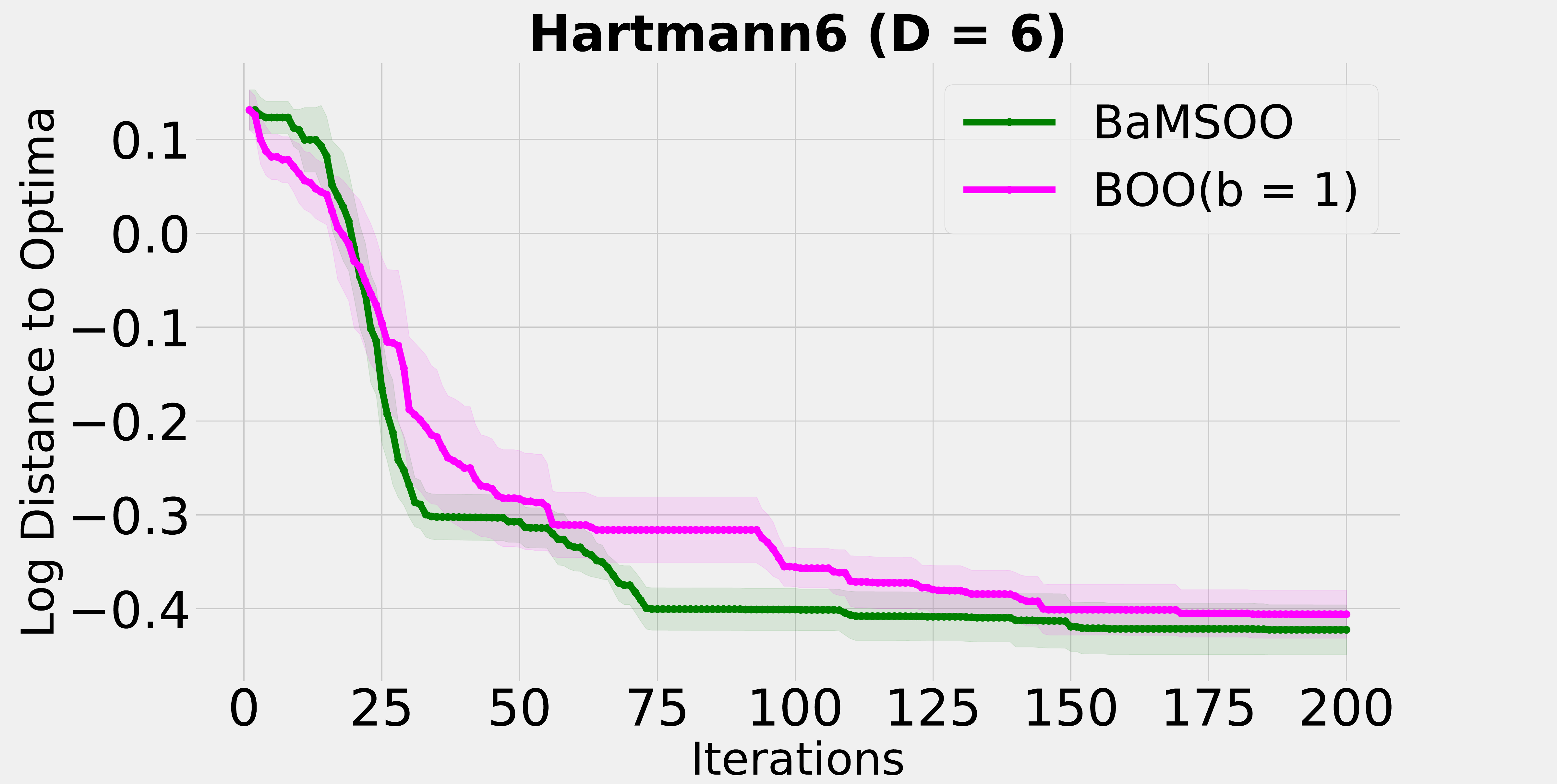}
}
\end{figure}

\end{document}